%% file: main.tex
\documentclass[11pt]{article}

\input{misc/template}
\input{misc/macro}

\begin{document}

\title{Unbounded Density Ratio Estimation and \\ Its Application to Covariate Shift Adaptation}

\author{Ren-Rui Liu\( ^{1} \) \and Jun Fan\( ^{2} \) \and Lei Shi\( ^{3} \) \and Zheng-Chu Guo\( ^{1, \ast} \)}

\footnotetext[1]{School of Mathematical Sciences, Zhejiang University, Hangzhou 310058, China}
\footnotetext[2]{Department of Mathematics, Hong Kong Baptist University, Kowloon, Hong Kong}
\footnotetext[3]{School of Mathematical Sciences and Shanghai, Key Laboratory for Contemporary Applied Mathematics, Fudan University, Shanghai 200433, China}
\renewcommand{\thefootnote}{\fnsymbol{footnote}}
\footnotetext[1]{The corresponding author is Zheng-Chu Guo. Email address: \href{mailto:guozc@zju.edu.cn}{guozc@zju.edu.cn}.}
\renewcommand{\thefootnote}{\arabic{footnote}}

\date{}

\maketitle

\begin{abstract}
    This paper focuses on the problem of unbounded density ratio estimation---an understudied yet critical challenge in statistical learning---and its application to covariate shift adaptation. Much of the existing literature assumes that the density ratio is either uniformly bounded or unbounded but known exactly. These conditions are often violated in practice, creating a gap between theoretical guarantees and real-world applicability. In contrast, this work directly addresses unbounded density ratios and integrates them into importance weighting for effective covariate shift adaptation. We propose a three-step estimation method that leverages unlabeled data from both the source and target distributions: (1) estimating a relative density ratio; (2) applying a truncation operation to control its unboundedness; and (3) transforming the truncated estimate back into the standard density ratio. The estimated density ratio is then employed as importance weights for regression under covariate shift. We establish rigorous, non-asymptotic convergence guarantees for both the proposed density ratio estimator and the resulting regression function estimator, demonstrating optimal or near-optimal convergence rates. Our findings offer new theoretical insights into density ratio estimation and learning under covariate shift, extending classical learning theory to more practical and challenging scenarios.

\end{abstract}

\input{text/1_introduction}

\input{text/2_result}

\input{text/3_discussion}

\input{text/4_proof}

\input{text/5_acknowledgment}

\appendix
\input{text/6_lemma}

\bibliography{reference}

\end{document}

%% file: misc/template.tex
\usepackage{amsmath}
\usepackage{amssymb}
\usepackage{amsthm}
\usepackage{mathtools}
\usepackage{mathrsfs}

\usepackage{abstract}
\usepackage{appendix}
\usepackage{titlesec}

\usepackage{natbib}
\bibpunct{(}{)}{;}{a}{,}{,}

\usepackage{hyperref}
\hypersetup{
    colorlinks=true,
    linkcolor={blue},
    urlcolor={black},
    citecolor={blue},
    linktoc=all,
}

\usepackage[
    left=1in,
    right=1in,
    top=1in,
    bottom=1in,
    headheight=0pt,
    headsep=5pt
]{geometry}
\usepackage{caption}

\usepackage[shortlabels,inline]{enumitem}

\usepackage{array}
\usepackage{booktabs}

\usepackage{graphicx}

\usepackage{xparse}
\usepackage{xstring}

\makeatletter
\newtoks\version
\newtoks\institute
\renewcommand{\@maketitle}{%
    \newpage
    \null
    \vskip 2em%
    \begin{center}%
        {\LARGE\bfseries \@title \par}%
        \vskip 1.5em%
            {\large
                \lineskip .5em%
                \begin{tabular}[t]{c}%
                    \@author \\[1ex]
                \end{tabular}\par}%
        \the\institute%
        \vskip 0.5ex%
    \end{center}%
    \par
}
\makeatother

\titleformat{\section}[block]
{\normalfont\LARGE\bfseries}
{\thesection}{1em}{}
\titleformat{\subsection}[block]
{\normalfont\Large\bfseries}
{\thesubsection}{1em}{}
\titleformat{\subsubsection}[block]
{\normalfont\large\bfseries}
{\thesubsubsection}{1em}{}

\AtBeginDocument{
    \setlength{\jot}{8pt}
    \setlength{\abovedisplayskip}{8pt}
    \setlength{\belowdisplayskip}{8pt}
    \setlength{\abovedisplayshortskip}{6pt}
    \setlength{\belowdisplayshortskip}{6pt}

    \RequirePackage[flushmargin]{footmisc}
    \setlength{\footnotesep}{12pt}
}

\theoremstyle{plain}
\newtheorem{theorem}{Theorem}
\newtheorem{proposition}{Proposition}
\newtheorem{lemma}{Lemma}
\newtheorem{corollary}{Corollary}
\newtheorem{definition}{Definition}
\newtheorem{assumption}{Assumption}

\newtheorem*{remark}{\normalfont\textbf{Remark}}

\newcommand{\keywords}[1]{\vskip 2ex\par\noindent\normalfont{\bfseries Keywords:} #1}
\newcommand{\email}[1]{\href{mailto:#1}{\nolinkurl{#1}}}

%% file: misc/macro.tex
\newcommand{\ROM}[1]{\uppercase\expandafter{\romannumeral #1}}

\newcommand{\secref}[1]{Section~\ref{#1}}
\newcommand{\figref}[1]{Figure~\ref{#1}}
\newcommand{\tabref}[1]{Table~\ref{#1}}
\newcommand{\thmref}[1]{Theorem~\ref{#1}}
\newcommand{\proref}[1]{Proposition~\ref{#1}}
\newcommand{\lemref}[1]{Lemma~\ref{#1}}
\newcommand{\colref}[1]{Corollary~\ref{#1}}
\newcommand{\defref}[1]{Definition~\ref{#1}}
\newcommand{\aspref}[1]{Assumption~\ref{#1}}

\newcommand{\neweqref}[2]{\textup{(\hyperref[#1]{\textbf{#2}})}}

\newenvironment{proofof}[1]{
    \begin{proof}[\normalfont\textbf{Proof of #1}]
        }{
    \end{proof}
}

\newcommand{\dd}{\mathrm{d}}
\newcommand{\ee}{\mathrm{e}}
\newcommand{\eqspace}{\hphantom{{}={}\!}}  

\newcommand{\bbR}{\mathbb{R}}

\newcommand{\calH}{\mathcal{H}}
\newcommand{\calL}{\mathcal{L}}
\newcommand{\calN}{\mathcal{N}}
\newcommand{\calX}{\mathcal{X}}
\newcommand{\calY}{\mathcal{Y}}

\newcommand{\bsy}{\boldsymbol{y}}

\newcommand{\LE}{\left}
\newcommand{\MID}{\middle}
\newcommand{\RI}{\right}

\newcommand{\ind}[1]{\mathbf{1}_{ #1 }}

\makeatletter
\DeclarePairedDelimiter{\DVert}{\lVert}{\rVert}
\DeclarePairedDelimiter{\Dangle}{\langle}{\rangle}
\DeclarePairedDelimiter{\Dbrace}{\lbrace}{\rbrace}
\DeclarePairedDelimiter{\Dbrack}{\lbrack}{\rbrack}
\DeclarePairedDelimiter{\Dparen}{\lparen}{\rparen}
\newcommand{\norm}[1]{%
  \if@display
    \DVert*{#1}%
  \else
    \DVert{#1}%
  \fi
}
\newcommand{\inner}[1]{%
  \if@display
    \Dangle*{#1}%
  \else
    \Dangle{#1}%
  \fi
}
\newcommand{\family}[1]{%
  \if@display
    \Dbrace*{#1}%
  \else
    \Dbrace{#1}%
  \fi
}
\newcommand{\Tr}[1]{%
  \if@display
    \operatorname{Tr} \Dparen*{#1}%
  \else
    \operatorname{Tr} \Dparen{#1}%
  \fi
}
\NewDocumentCommand{\EE}{o m}{%
  \if@display
    \IfNoValueTF{#1}
      {\operatorname{\mathbb{E}} \Dbrack*{#2}}
      {\operatorname{\mathbb{E}}_{#1} \Dbrack*{#2}}%
  \else
    \IfNoValueTF{#1}
      {\operatorname{\mathbb{E}} \Dbrack{#2}}
      {\operatorname{\mathbb{E}}_{#1} \Dbrack{#2}}%
  \fi
}
\NewDocumentCommand{\PP}{o m}{%
  \if@display
    \IfNoValueTF{#1}
      {\operatorname{\mathrm{P}} \Dparen*{#2}}
      {\operatorname{\mathrm{P}}_{#1} \Dparen*{#2}}%
  \else
    \IfNoValueTF{#1}
      {\operatorname{\mathrm{P}} \Dparen{#2}}
      {\operatorname{\mathrm{P}}_{#1} \Dparen{#2}}%
  \fi
}
\makeatother

\makeatletter
\newcommand{\biggg}{\bBigg@\thr@@}
\newcommand{\Biggg}{\bBigg@{3.5}}
\makeatother

\let\oldforall\forall
\renewcommand{\forall}{\mathrel{\oldforall}}
\let\oldexists\exists
\renewcommand{\exists}{\mathrel{\oldexists}}

\DeclareMathOperator*{\argmin}{arg\,min}

\newcommand{\rhoS}{\rho^{\textnormal{S}}}
\newcommand{\rhoT}{\rho^{\textnormal{T}}}
\newcommand{\rhoR}{\rho^{\textnormal{R}}}
\newcommand{\rhoSX}{\rhoS_{\calX}}
\newcommand{\rhoTX}{\rhoT_{\calX}}
\newcommand{\rhoRX}{\rhoR_{\calX}}

\newcommand{\nU}{n_{\theta}}
\newcommand{\nL}{n_{f}}

\newcommand{\LS}{L_{\textnormal{S}}}
\newcommand{\hLS}{\widehat{L}_{\textnormal{S}}}
\newcommand{\LT}{L_{\textnormal{T}}}
\newcommand{\hLT}{\widehat{L}_{\textnormal{T}}}
\newcommand{\LTlam}{L_{\textnormal{T}, \lambda}}
\newcommand{\LR}{L_{\textnormal{R}}}
\newcommand{\hLR}{\widehat{L}_{\textnormal{R}}}
\newcommand{\LRmu}{L_{\textnormal{R}, \mu}}
\newcommand{\hLRmu}{\widehat{L}_{\textnormal{R}, \mu}}
\newcommand{\LW}{L_{\textnormal{W}}}
\newcommand{\hLW}{\widehat{L}_{\textnormal{W}}}
\newcommand{\LWlam}{L_{\textnormal{W}, \lambda}}
\newcommand{\hLWlam}{\widehat{L}_{\textnormal{W}, \lambda}}

\newcommand{\gmu}{g_{\mu}}
\newcommand{\glam}{g_{\lambda}}

\newcommand{\phimu}{\phi_{\mu}}
\newcommand{\hphi}{\widehat{\phi}}
\newcommand{\hphimu}{\hphi_{\mu}}
\newcommand{\hphimuD}{\hphi_{\mu, D}}

\newcommand{\htheta}{\widehat{\theta}}

\newcommand{\hthetamuD}{\htheta_{\mu, D}}

\newcommand{\Kx}{K_{x}}
\newcommand{\Kxp}{K_{x'}}

\newcommand{\frho}{f_{\rho}}
\newcommand{\flam}{f_{\lambda}}
\newcommand{\hflam}{\widehat{f}_{\lambda}}

\newcommand{\urho}{u_{\rho}}

\newcommand{\hSWa}{\widehat{S}_{\textnormal{W}}^{\ast}}

\newcommand{\ranH}[1]{[\calH]^{#1}}

%% file: text/1_introduction.tex
\section{Introduction}
\label{sec: introduction}

Supervised learning builds predictive models using labeled training data drawn from a source domain. Classical statistical learning theory often assumes that this training data is representative of the target domain, i.e., the environment where the model will be deployed \citep{Vapnik1998StatisticalLT}. However, this assumption rarely holds in real-world applications, where the data distribution frequently shifts between the source domain and the target domain. Temporal evolution, sampling biases, and contextual variations frequently induce mismatches between the source and target data distributions, thereby challenging the generalization ability of learned models. This paper focuses specifically on \emph{covariate shift}, a common distribution shift scenario in which the input (covariate) distributions differ between the source and target domains, while the conditional distribution of the output given the input remains invariant. Such shifts arise naturally in practice; for example, medical datasets collected from different hospitals may exhibit varying demographic compositions (covariate distributions), while the diagnostic relationship between symptoms and disease (conditional distribution) remains unchanged \citep{Finlayson2021ClinicianDS}.

To mitigate the impact of covariate shift, \emph{importance weighting} \citep{Shimodaira2000ImprovingPI} has become a widely adopted strategy. It compensates for distributional discrepancy by reweighting source data points according to the \emph{density ratio}, formally defined as the Radon--Nikodym derivative of the target distribution with respect to the source distribution. Extensive research has been devoted to estimating these density ratios \citep{Sugiyama2012DensityRE, Gizewski2022RegularizationUD, Feng2024DeepNQ, Xu2025EstimatingUD, Zheng2026ErrorAD, Gizewski2026ImpactSK}, and many theoretical analyses have examined the behavior of importance-weighted learning procedures \citep{Gogolashvili2023WhenIW, Fan2025SpectralAU, Liu2025SpectralAM, Ma2023OptimallyTC, Gizewski2022RegularizationUD, Gizewski2026ImpactSK}. Nevertheless, a large portion of existing work relies on the assumption that the density ratio is either uniformly bounded, or unbounded but known exactly. Both conditions frequently fail in practical settings, creating a significant gap between theoretical guarantees and real-world applications. In this work, we aim to estimate unbounded density ratios and incorporate them into importance weighting to achieve effective covariate shift adaptation.

We formulate our analysis within a regression setting employing squared loss. Consider a labeled training dataset \( \family{ (x_{i}, y_{i}) }_{i=1}^{\nL} \) sampled independently from an unknown source distribution \( \rhoS_{\calX \times \calY} \), where \( x_{i} \in \calX \) denotes the input from a compact metric space and \( y_{i} \in \calY \subseteq \bbR \) represents the output. Our objective is to learn a predictor that performs well on a target distribution. For a predictor \( f \), its target expected risk is defined as
\begin{equation}
    \label{eq: target expected risk}
    \EE[(x, y) \sim \rhoT_{\calX \times \calY}]{(y - f(x))^{2}},
\end{equation}
where \( \rhoT_{\calX \times \calY} \) denotes the joint target distribution over inputs and outputs. It is well-known that the minimizer of this risk over all measurable functions is the regression function
\[
    \frho(x) = \int y \, \dd \rho_{\calY \mid \calX}(y \mid x),
\]
where \( \rho_{\calY \mid \calX} \) is the conditional distribution of \( y \) given \( x \), and \( \rhoT_{\calX \times \calY} = \rho_{\calY \mid \calX} \cdot \rhoTX \). Hence, estimating \( \frho \) becomes our primary goal. Since the target expected risk is inaccessible due to the unknown distribution \( \rhoT_{\calX \times \calY} \), we employ the empirical risk as a surrogate:
\begin{equation}
    \label{eq: raw empirical risk}
    \frac{1}{\nL} \sum_{i=1}^{\nL} (y_{i} - f(x_{i}))^{2}.
\end{equation}
In the standard setting where \( \rhoS_{\calX \times \calY} = \rhoT_{\calX \times \calY} \), this empirical risk \eqref{eq: raw empirical risk} is an unbiased estimator of the target expected risk \eqref{eq: target expected risk}, and minimizing it typically yields a consistent estimator of \( \frho \). Under \emph{covariate shift}, however, the conditional distribution \( \rho_{\calY \mid \calX} \) remains invariant, but the marginal input distributions differ. That is,
\[
    \rhoS_{\calX \times \calY} = \rho_{\calY \mid \calX} \cdot \rhoSX,
    \quad \rhoT_{\calX \times \calY} = \rho_{\calY \mid \calX} \cdot \rhoTX
\]
with \( \rhoSX \ne \rhoTX \). Consequently, the empirical risk \eqref{eq: raw empirical risk} becomes a biased estimator, and its minimizer may generalize poorly to \( \rhoT_{\calX \times \calY} \). To correct this bias, we adopt the \emph{importance weighting} strategy \citep{Shimodaira2000ImprovingPI}, which reweights each source data point to align the source and target distributions:
\begin{equation}
    \label{eq: modified empirical risk}
    \frac{1}{\nL} \sum_{i=1}^{\nL} \theta(x_{i}) (y_{i} - f(x_{i}))^{2}.
\end{equation}
Here, \( \theta = \dd \rhoTX / \dd \rhoSX \) is the Radon--Nikodym derivative (commonly referred to as the \emph{density ratio}) of the target with respect to the source input distribution. This modified empirical risk \eqref{eq: modified empirical risk} is again an unbiased estimator of the target expected risk \eqref{eq: target expected risk}.

In practice,  \( \theta \) is typically unknown and must be estimated from available data. Suppose that, in addition to the labeled training dataset \( \family{ (x_{i}, y_{i}) }_{i=1}^{\nL} \), we have access to unlabeled samples \( \family{ x_{j} }_{j=1}^{\nU} \) and \( \family{ x'_{j} }_{j=1}^{\nU} \) drawn from \( \rhoSX \) and \( \rhoTX \), respectively. Using these data, we construct an estimator \( \htheta \) for the density ratio (see \secref{sec: result} for details), leading to the importance-weighted empirical risk
\begin{equation}
    \label{eq: importance-weighted empirical risk}
    \frac{1}{\nL} \sum_{i=1}^{\nL} \htheta(x_{i}) (y_{i} - f(x_{i}))^{2}.
\end{equation}
Note that density ratios can be unbounded, making their estimation challenging. For example, if \( \calX = [0, 1] \), with \( \rhoSX \) uniform and \( \rhoTX \) having density \( -\log(x) \), then \( \theta(x) = -\log(x) \to \infty \) as \( x \to 0 \).

We minimize the importance-weighted empirical risk \eqref{eq: importance-weighted empirical risk} to obtain an estimator of the regression function \( \frho \). A popular choice is kernel ridge regression (KRR). Let \( K \colon \calX \times \calX \to \bbR \) be a continuous, symmetric, positive semi-definite Mercer kernel, which generates a reproducing kernel Hilbert space (RKHS) \( \calH \) with the reproducing property:
\[
    f(x) = \inner{ f, K(\cdot, x) }_{\calH},
    \quad \forall f \in \calH.
\]
We denote the uniform bound of the kernel over the compact input space \( \calX \) as \( \sup_{x \in \calX} K(x, x) = \kappa^{2} < \infty \). The KRR estimator is then defined as
\[
    \hflam^{\text{krr}} = \argmin_{f \in \calH} \frac{1}{\nL} \sum_{i=1}^{\nL} \htheta(x_{i}) (y_{i} - f(x_{i}))^{2} + \lambda \norm{ f }_{\calH}^{2}.
\]
As shown in prior work (e.g., \citealp{Vito2005RiskBR}), this problem admits the closed-form solution
\begin{equation}
    \label{eq: KRR}
    \hflam^{\text{krr}} = (\hLW + \lambda I)^{-1} \, \hSWa \, \bsy,
\end{equation}
where \( \hLW \) is the importance-weighted empirical integral operator:
\[
    \hLW f = \frac{1}{\nL} \sum_{i=1}^{\nL} \htheta(x_{i}) f(x_{i}) \, K(\cdot, x_{i}),
\]
and \( \hSWa \) is the adjoint of the importance-weighted sampling operator:
\[
    \hSWa \, \bsy = \frac{1}{\nL} \sum_{i=1}^{\nL} \htheta(x_{i}) y_{i} \, K(\cdot, x_{i}),
    \quad \bsy = (y_{1}, \dots, y_{\nL})^{\top}.
\]

The KRR estimator \eqref{eq: KRR} belongs to a broader family of regularized learning algorithms known as \emph{spectral algorithms}, which are formally defined in \eqref{eq: spectral algorithm}. These algorithms, originally developed for solving ill-posed linear inverse problems \citep{Engl2015RegularizationIP}, have been adapted to the regression setting by exploiting the connections between learning theory and inverse problems \citep{Vito2005LearningFE, Gerfo2008SpectralAS, Bauer2007RegularizationAL}. Spectral algorithms construct estimators via carefully designed \emph{filter functions}. These functions perform regularization by selectively amplifying the dominant eigencomponents of the integral operator while suppressing the less influential ones.
\begin{definition}[Filter functions]
    \label{def: filter}
    A family of functions \( \glam \colon [0, \kappa^{2}] \to [0, \infty) \), parameterized by \( \lambda > 0 \), constitutes filter functions if:
    \begin{itemize}
        \item There exists \( E \ge 0 \) such that for all \( c \in [0, 1] \):
              \begin{equation}
                  \label{eq: glam}
                  \sup_{t \in [0, \kappa^{2}]} t^{c} \glam(t) \le E \lambda^{c - 1}.
              \end{equation}
        \item There exist \( \tau \ge 1 \) and \( F \ge 0 \) such that for all \( c \in [0, \tau] \):
              \begin{equation}
                  \label{eq: 1 - t glam}
                  \sup_{t \in [0, \kappa^{2}]} t^{c} |1 - t \glam(t)| \le F \lambda^{c}.
              \end{equation}
    \end{itemize}
\end{definition}
Condition \eqref{eq: glam} ensures the regularized inverse remains bounded, guaranteeing numerical stability. Condition \eqref{eq: 1 - t glam} controls the approximation error by requiring the residual \( |1 - t \glam(t)| \) to decay at a rate governed by \( \lambda \). The parameter \( \tau \), known as the qualification of the regularization method, determines the maximum degree of source smoothness that the algorithm can effectively handle. For a given filter function \( \glam \), the importance-weighted spectral algorithm associated with some \( \htheta \) is defined as
\begin{equation}
    \label{eq: spectral algorithm}
    \hflam = \glam(\hLW) \, \hSWa \, \bsy,
    \quad \bsy = (y_{1}, \dots, y_{\nL})^{\top},
\end{equation}
where \( \glam \) acts on the eigenvalues of \( \hLW \), treating it as a positive definite operator on \( \calH \). We note that the unboundedness of the density ratio \( \htheta \) may lead to unbounded \( \hLW \). However, under appropriate restrictions on the density ratio (e.g., \aspref{asp: moment theta} and \aspref{asp: source phi}), the upper bound for the eigenvalues of \( \hLW \) converges in probability to \( \kappa^{2} \) as the sample sizes \( \nU, \nL \to \infty \). Therefore, we restrict our analysis to sampling scenarios where \( \glam(\hLW) \) is well-defined, and establish our convergence results with high probability.

The spectral algorithm framework encompasses various regularization methods, including:
\begin{itemize}
    \item Kernel ridge regression:
          \( \glam^{\text{krr}}(t) = (t + \lambda)^{-1} \), with qualification \( \tau = 1 \) and constants \( E = F = 1 \); here \( \lambda \) serves as the regularization parameter.

    \item Early-stopped gradient flow:
          \( \glam^{\text{gf}}(t) = t^{-1}(1 - \ee^{-t/\lambda}) \), which achieves arbitrary qualification \( \tau \ge 1 \) with \( E = 1, F = (\tau / \ee)^{\tau} \); the stopping time corresponds to \( 1/\lambda \).

    \item Spectral cutoff:
          \( \glam^{\text{cut}}(t) = t^{-1} \ind{t \ge \lambda} \), with arbitrary \( \tau \ge 1 \) and \( E = F = 1 \); the cutoff threshold is \( \lambda \).
\end{itemize}

This paper addresses the challenges of covariate shift by proposing an integrated framework that combines density ratio estimation with importance-weighted spectral algorithms. Our main contributions are summarized as follows:
\begin{enumerate}
    \item We develop a three-step procedure to estimate density ratios that may be unbounded. This approach first estimates the relative density ratio \citep{Yamada2013RelativeDE}, then applies truncation to handle unboundedness, and finally transforms it back to the standard density ratio. The estimated density ratio is then employed as importance weights for regression under covariate shift.

    \item We establish convergence guarantees for both the density ratio estimator and the final regression function estimator. Our analysis explicitly characterizes how the accuracy of the estimated density ratio governs the overall performance of covariate shift adaptation. A key assumption is that the unlabeled sample size, used for density ratio estimation, grows polynomially with the labeled sample size used for regression. Under this condition, the regression function estimator achieves near-optimal convergence rates within an interval of smoothness parameters whose width is determined by the polynomial order; a higher order expands this optimality range. These results provide new insights into the interplay between density ratio estimation accuracy and learning performance under covariate shift.
\end{enumerate}

The remainder of this paper is organized as follows: \secref{sec:  result} presents the necessary assumptions, describes our method for estimating unbounded density ratios, and states the main theoretical results; \secref{sec: discussion} provides a literature review and comparative analysis with existing works; \secref{sec: proof} contains the proofs of the main theorems, with auxiliary lemmas deferred to \hyperref[sec: lemma]{Appendix}.

%% file: text/2_result.tex
\section{Main Results}
\label{sec: result}

In this section, we describe our method for estimating an unbounded density ratio. We then incorporate this estimated density ratio into an importance weighting scheme to learn the regression function. We establish convergence guarantees for both the density ratio estimator and the regression function estimator.

\subsection{Unbounded Density Ratio Estimation}

In the following, we estimate the unbounded density ratio \( \theta \) using unlabeled data \( \family{ x_{j} }_{j=1}^{\nU} \) from the source distribution \( \rhoSX \) and \( \family{ x'_{j} }_{j=1}^{\nU} \) from the target distribution \( \rhoTX \). We begin by presenting key assumptions under which the density ratio estimator is consistent and well-behaved. First, we impose a moment condition on the density ratio \( \theta \) with respect to the source distribution \( \rhoSX \).
\begin{assumption}[Moment condition of \( \theta \)]
    \label{asp: moment theta}
    Let \( \theta = \dd \rhoTX / \dd \rhoSX \) denote the density ratio. There exists \( M \in (2, \infty] \) such that for all \( 2 < m \le M \),
    \[
        \int_{\calX} \theta^{m}(x) \, \dd \rhoSX(x) \le \varXi_{m} < \infty
    \]
    holds for some \( \varXi_{m} > 0 \).
\end{assumption}
\aspref{asp: moment theta} encompasses most existing work on density ratio estimation and importance weighting analysis. This includes cases where the density ratio is uniformly bounded \citep{Sugiyama2012DensityRE, Gizewski2022RegularizationUD, Myleiko2025RecoveringRD, Gizewski2026ImpactSK}, sub-exponential \citep{Xu2025EstimatingUD}, or satisfies the Bernstein-type condition \citep{Gogolashvili2023WhenIW, Fan2025SpectralAU, Liu2025SpectralAM}, which correspond to \( M = \infty \) in our assumption. It also covers cases where the density ratio has a finite moment of order higher than 2 \citep{Feng2024DeepNQ}, corresponding to \( 2 < M < \infty \).

Despite the moment condition, directly estimating an unbounded density ratio within a reproducing kernel Hilbert space remains problematic, since RKHSs consist exclusively of bounded functions. To address this limitation, we employ the concept of the \emph{relative density ratio} \citep{Yamada2013RelativeDE}. The core idea is to introduce a mixture of the source and target distributions. Specifically, for a mixing coefficient \( \alpha \in (0, 1) \), we define the mixture distribution \( \rhoRX \) as
\[
    \rhoRX = (1 - \alpha) \rhoSX + \alpha \rhoTX.
\]
The relative density ratio \( \phi \) is then defined as
\begin{equation}
    \label{eq: phi}
    \phi(x)
    = \frac{\dd \rhoTX(x)}{\dd \rhoRX(x)}
    = \frac{\theta(x)}{\alpha \theta(x) + (1 - \alpha)}.
\end{equation}
A key advantage of this formulation is that \( \phi \) is bounded in \( [0, 1 / \alpha] \), making it a more suitable candidate for estimation within an RKHS. Furthermore, we can recover the standard density ratio:
\begin{equation}
    \label{eq: recover theta}
    \theta(x) = \frac{(1 - \alpha) \phi(x)}{1 - \alpha \phi(x)}.
\end{equation}
As illustrated in \figref{fig: relative}, the relative density ratio typically exhibits smoother behavior compared to the standard density ratio, effectively mitigating sharp peaks and heavy tails.

\begin{figure}
    \centering
    \includegraphics[width=2.5in, height=2.5in]{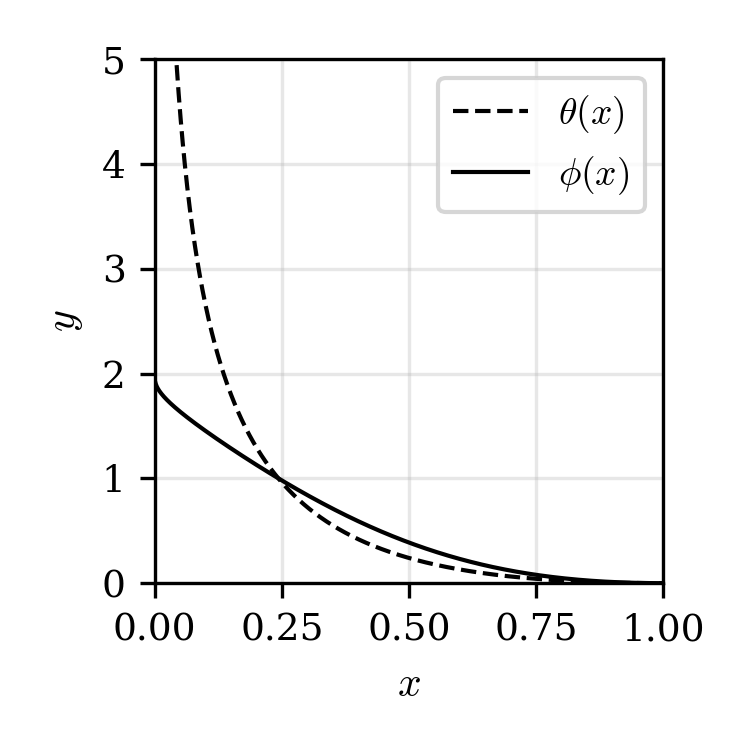}
    \includegraphics[width=2.5in, height=2.5in]{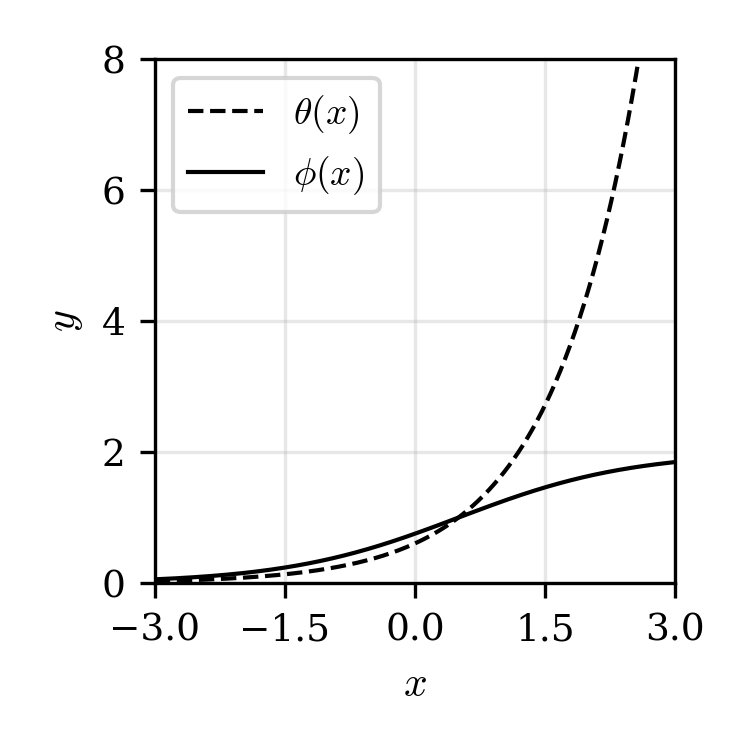}
    \caption{A comparison between the standard density ratio \( \theta \) and the relative density ratio \( \phi \). The mixing coefficient is set to \( \alpha = 1 / 2 \). Left: the source distribution \( \rhoSX \) is uniform on \( [0, 1] \), while the target distribution \( \rhoTX \) has a density proportional to \( (\log x)^{2} \). Right: the source and target distributions are \( \calN(0, 1) \) and \( \calN(1, 1) \), respectively.}
    \label{fig: relative}
\end{figure}

Our strategy is to estimate the unbounded density ratio \( \theta \) indirectly through a bounded relative density ratio \( \phi \) in the RKHS \( \calH \). To this end, we first define the necessary operators and then develop the estimation procedure. First, we introduce the integral operator \( \LR \) with respect to the mixture distribution \( \rhoRX \):
\[
    \LR \, f = \int_{\calX} f(x) \, K(\cdot, x) \, \dd \rhoRX(x).
\]
The relative density ratio \( \phi \) satisfies the following key relation:
\begin{equation}
    \LR \, \phi
    = \int_{\calX} \phi(x) \, K(\cdot, x) \, \dd \rhoRX(x)
    = \int_{\calX} K(\cdot, x) \, \dd \rhoTX(x)
    = \LT \, \ind{\calX},
\end{equation}
where \( \ind{\calX} \) denotes the constant unit function, and \( \LT \) denotes the integral operator with respect to \( \rhoTX \):
\[
    \LT \, f = \int_{\calX} f(x) \, K(\cdot, x) \, \dd \rhoTX(x).
\]
Given samples \( \family{ x_{j} }_{j=1}^{\nU} \) and \( \family{ x'_{j} }_{j=1}^{\nU} \), we construct empirical versions of these operators:
\begin{align*}
    \hLR \, f & = \frac{1 - \alpha}{\nU} \sum_{j=1}^{\nU} f(x_{j}) \, K(\cdot, x_{j}) + \frac{\alpha}{\nU} \sum_{j=1}^{\nU} f(x'_{j}) \, K(\cdot, x'_{j}), \\
    \hLT \, f & = \frac{1}{\nU} \sum_{j=1}^{\nU} f(x'_{j}) \, K(\cdot, x'_{j}).
\end{align*}
Using these empirical operators and a filter function \( \gmu \) from \defref{def: filter} with regularization parameter \( \mu > 0 \), we estimate \( \phi \) via
\begin{equation}
    \label{eq: hphimu}
    \hphimu = \gmu(\hLR) \, \hLT \, \ind{\calX}.
\end{equation}
This estimator can be interpreted as performing kernel mean matching \citep{Gretton2008CovariateSK, Gizewski2022RegularizationUD, Gizewski2026ImpactSK} to estimate the density ratio between \( \rhoSX \) and \( \rhoRX \). Since the true relative density ratio \( \phi \) is nonnegative, we refine our estimate by taking the positive part; to handle potential unboundedness, we apply truncation with threshold \( D > 0 \):
\begin{equation}
    \label{eq: hphimuD}
    \hphimuD(x) = \min \left\{ \max \family{ \hphimu(x), 0}, \frac{D}{\alpha D + (1 - \alpha)} \right\}.
\end{equation}
Note that \( D / (\alpha D + (1 - \alpha)) < 1 / \alpha \). Finally, we recover the standard density ratio using the inverse transformation \eqref{eq: recover theta}:
\begin{equation}
    \label{eq: hthetamuD}
    \hthetamuD(x) = \frac{(1 - \alpha) \hphimuD(x)}{1 - \alpha \hphimuD(x)}.
\end{equation}
The resulting estimator \( \hthetamuD \) is bounded in \( [0, D] \). The truncation threshold \( D \) will be chosen as a function of the sample size \( \nU \) to balance the error induced by regularization and by truncation.

To analyze the convergence behavior of our estimator, we need to characterize the smoothness of the relative density ratio \( \phi \). This is formalized through the following source condition.
\begin{assumption}[Source condition of \( \phi \)]
    \label{asp: source phi}
    Let \( \tau \) be the qualification parameter from \defref{def: filter}, and \( \phi = \dd \rhoTX / \dd \rhoRX \) be the relative density ratio. There exists \( \iota \in [1/2, \tau] \) such that
    \[
        \phi = \LR^{\iota} \, \varpi.
    \]
    holds for some \( \varpi \in \calL^{2}(\calX, \rhoRX) \).
\end{assumption}
This source condition is standard in the learning theory literature \citep{Smale2003EstimatingAE, Cucker2007LearningTA, Caponnetto2007OptimalRR}. To understand its implications, recall that for any nondegenerate finite measure \( \rho_{\calX} \) on \( \calX \), the associated integral operator \( L_{\rho_{\calX}} \) satisfies an important embedding property: according to \citet[Theorem~4.12]{Cucker2007LearningTA}, the operator \( L_{\rho_{\calX}}^{1/2} \) maps \( \calL^{2}(\calX, \rho_{\calX}) \) into the RKHS \( \calH \), with
\begin{equation}
    \label{eq: rho-norm / H-norm}
    \norm{ L_{\rho_{\calX}}^{1/2} \, f }_{\calH} = \norm{ f }_{\rho_{\calX}},
    \quad \forall f \in \calL^{2}(\calX, \rho_{\calX}),
\end{equation}
where \( \norm{ \cdot }_{\rho_{\calX}} \) denotes the norm in \( \calL^{2}(\calX, \rho_{\calX}) \). The lower bound \( \iota \ge 1/2 \) in \aspref{asp: source phi} thus ensures that \( \phi \) belongs to \( \calH \). Furthermore, using the definition of the mixture distribution \( \rhoRX \), we can bound the norm of \( \varpi \) as follows:
\begin{align*}
    \norm{ \varpi }_{\rhoRX}^{2}
     & = \int_{\calX} \varpi^{2}(x) \, \dd \rhoRX(x)
    = \int_{\calX} \varpi^{2}(x) \, ((1 - \alpha) \dd \rhoSX(x) + \alpha \dd \rhoTX(x))  \\
     & = (1 - \alpha) \norm{ \varpi }_{\rhoSX}^{2} + \alpha \norm{ \varpi }_{\rhoTX}^{2}
    \le \max \family{ \norm{ \varpi }_{\rhoSX}^{2}, \norm{ \varpi }_{\rhoTX}^{2} }.
\end{align*}

With these preliminaries established, we now present our first main result, which characterizes the convergence behavior of the relative density ratio estimator \( \hphimu \) and shows that its convergence rate is optimal in the minimax sense \citep{Caponnetto2007OptimalRR}.
\begin{theorem}
    \label{thm: relative density ratio}
    Suppose that \aspref{asp: source phi} holds with \( \iota \in [1/2, \tau] \), and set the regularization parameter as \( \mu = \nU^{-\varsigma} \) with \( \varsigma = 1 / (2 \iota + 1) \). Then, for any \( \delta \in (0, 1) \), with probability at least \( 1 - \delta \), when the sample size \( \nU \) is sufficiently large such that
    \begin{equation}
        \label{eq: sample size condition 1}
        \begin{aligned}
            \nU \ge 85 \kappa^{2} \LE( 1 + \log \frac{6 \kappa^{2} (\norm{ \LS } + \norm{ \LT } + 1)}{\min \family{ \norm{ \LS }, \norm{ \LT } } \cdot \delta} + \frac{2}{1 - \varsigma} \RI)^{\frac{2}{1 - \varsigma}},
        \end{aligned}
    \end{equation}
    with \( \norm{ \cdot } \) denoting the operator norm \( \norm{ \cdot }_{\calH \to \calH} \) on the RKHS \( \calH \), the relative density ratio estimator defined in \eqref{eq: hphimu} satisfies
    \[
        \norm{ \phi - \hphimu }_{\rhoRX} \le \varDelta_{\phi} \nU^{-\frac{\iota}{2 \iota + 1}} \log \frac{18}{\delta},
    \]
    where \( \varDelta_{\phi} \) is a constant given in \eqref{eq: Delta_phi}, being independent of \( \alpha \) (the mixing coefficient), \( \nU \), or \( \delta \).
\end{theorem}

Our final density ratio estimator \( \hthetamuD \) is obtained from \( \hphimu \) via relations \eqref{eq: hphimuD} and \eqref{eq: hthetamuD}. Its convergence properties are quantified by the following theorem.
\begin{theorem}
    \label{thm: density ratio}
    Suppose that \aspref{asp: moment theta} holds with \( M \in (2, \infty] \) and \aspref{asp: source phi} holds with \( \iota \in [1/2, \tau] \). For a fixed \( m \in (2, M] \), set the regularization parameter as \( \mu = \nU^{-\varsigma} \) and the truncation threshold as \( D = \nU^{\nu} \), where
    \[
        \varsigma = \frac{1}{2 \iota + 1},
        \quad \nu = \frac{1}{m} \cdot \frac{2 \iota}{2 \iota + 1}.
    \]
    Then, for any \( \delta \in (0, 1) \), with probability at least \( 1 - \delta \), when the sample size \( \nU \) satisfies condition \eqref{eq: sample size condition 1}, the density ratio estimator given in \eqref{eq: hthetamuD} achieves the bound
    \[
        \norm{ \theta - \hthetamuD }_{\rhoSX} \le \LE( (1 - \alpha)^{-3/2} \varDelta_{\phi} + \varXi_{m}^{1/2} \RI) \nU^{- \LE( \frac{\iota}{2 \iota + 1} - 2 \nu \RI)} \log \frac{18}{\delta}.
    \]
\end{theorem}
As the moment parameter \( m \) increases, the convergence rate established in \thmref{thm: density ratio} improves and approaches the order \( O(\nU^{-\iota / (2 \iota + 1)}) \). When all moments of the density ratio are finite (i.e., \( M = \infty \) in \aspref{asp: moment theta}), this rate can be made arbitrarily close to the limiting value.

\subsection{Covariate Shift Adaptation}

In the previous section, we derived an estimator \( \hthetamuD \) for the unbounded density ratio \( \theta \). We now incorporate this estimator into importance weighting to achieve effective covariate shift adaptation. From this point forward, we fix the weighting function as \( \htheta = \hthetamuD \), where the regularization parameter \( \mu \) and the truncation threshold \( D \) are chosen optimally as specified in \thmref{thm: density ratio}. As discussed earlier, given labeled training data \( \family{ (x_{i}, y_{i}) }_{i=1}^{\nL} \) drawn from the source distribution \( \rhoS_{\calX \times \calY} \) and a filter function \( \glam \) from \defref{def: filter} with regularization parameter \( \lambda > 0 \), we estimate the regression function \( \frho \) through the importance-weighted spectral algorithm \eqref{eq: spectral algorithm} with \( \htheta = \hthetamuD \). With a slight abuse of notation, we reuse the expression from \eqref{eq: spectral algorithm}: 
\begin{equation}
    \label{eq: hflam}
    \hflam = \glam(\hLW) \, \hSWa \, \bsy,
    \quad \bsy = (y_{1}, \dots, y_{\nL})^{\top},
\end{equation}
where the importance-weighted empirical integral operator \( \hLW \) and the adjoint sampling operator \( \hSWa \) are defined as
\[
    \hLW f = \frac{1}{\nL} \sum_{i=1}^{\nL} \htheta(x_{i}) f(x_{i}) \, K(\cdot, x_{i}),
    \quad \hSWa \, \bsy = \frac{1}{\nL} \sum_{i=1}^{\nL} \htheta(x_{i}) y_{i} \, K(\cdot, x_{i}).
\]

Analogous to \aspref{asp: source phi}, we impose the following source condition to characterize the smoothness of the regression function \( \frho \).
\begin{assumption}[Source condition of \( \frho \)]
    \label{asp: source frho}
    Let \( \tau \) be the qualification parameter from \defref{def: filter}. There exists \( r \in [1/2, \tau] \) such that
    \[
        \frho = \LT^{r} \, \urho
    \]
    holds for some \( \urho \in \calL^{2}(\calX, \rhoTX) \).
\end{assumption}
We also introduce a moment condition to bound the noise in the training data.
\begin{assumption}[Moment condition of the noise]
    \label{asp: moment noise}
    There exist positive constants \( G_{0} \) and \( \sigma_{0} \) such that
    \[
        \int_{\calX} |y - \frho(x)|^{\ell} \, \dd \rho_{\calY \mid \calX}(y \mid x) \le \frac{1}{2} \ell! G_{0}^{\ell-2} \sigma_{0}^{2},
        \quad \ell = 2, 3, \dots
    \]
    holds for any \( x \in \calX \).
\end{assumption}
This is a standard assumption in statistical learning theory to ensure fast decay of the noise tail probabilities \citep{Caponnetto2007OptimalRR, Fischer2020SobolevNL}. It holds, for example, when the noise is uniformly bounded, Gaussian, sub-Gaussian, or sub-exponential.

We now present our next theorem, which characterizes the convergence  behavior of the regression function estimator \( \hflam \).
\begin{theorem}
    \label{thm: importance weighting}
    For any \( \delta \in (0, 1) \), assume that the conditions of \thmref{thm: density ratio} are satisfied with \( M \in (2, \infty] \), \( m \in (2, M] \), and \( \iota \in [1/2, \tau] \), ensuring that the upper bound for \( \norm{ \theta - \hthetamuD }_{\rhoSX} \) holds with probability at least \( 1 - \delta / 2 \). Furthermore, suppose that \aspref{asp: source frho} holds with \( r \in [1/2, \tau] \), \aspref{asp: moment noise} holds with \( G_{0}, \sigma_{0} > 0 \), and set the regularization parameter as \( \lambda = \nL^{-s} \) with
    \[
        0 < s < \frac{1}{2 r + 1}.
    \]
    In addition to \eqref{eq: sample size condition 1}, assume that the sample sizes \( \nU \) and \( \nL \) satisfy the following constraints:
    \begin{equation}
        \label{eq: sample size condition 2}
        \LE\{
        \begin{aligned}
             & \nU^{\frac{\iota}{2 \iota + 1} \cdot \min \family{ 1, \frac{2}{2 r - 1} } - 2 \nu} \ge \frac{2 \kappa^{2}}{\min \family{ \norm{ \LT }, 1 }} \LE( (1 - \alpha)^{-3/2} \varDelta_{\phi} + \varXi_{m}^{1/2} \RI) \log \frac{36}{\delta} \cdot \nL^{s}, \\
             & \nL^{\frac{1}{2 r + 1} - s} \ge 14 \kappa^{2} \LE( \log \frac{24 \kappa^{2} (\norm{ \LT } + 2)}{\norm{ \LT } \delta} + 1 + \frac{1}{2 r} \RI) \nU^{2 \nu},
        \end{aligned}
        \RI.
    \end{equation}
    where \( \alpha \in (0, 1) \) is the mixing coefficient and \( \norm{ \cdot } \) denotes the operator norm \( \norm{ \cdot }_{\calH \to \calH} \) on the RKHS \( \calH \). Then, with probability at least \( 1 - \delta \), the regression function estimator \( \hflam \) defined in \eqref{eq: hflam} achieves the following bounds:
    \[
        \norm{ \frho - \hflam }_{\rhoTX} \le \varDelta_{f} \LE( (1 - \alpha)^{-3/2} + \varXi_{m}^{1/2} + 1 \RI) \nL^{-r s} \log \frac{36}{\delta},
    \]
    and
    \[
        \norm{ \frho - \hflam }_{\calH} \le \varDelta_{f} \LE( (1 - \alpha)^{-3/2} + \varXi_{m}^{1/2} + 1 \RI) \nL^{-\LE( r - \frac{1}{2} \RI) s} \log \frac{36}{\delta}.
    \]
    Here, \( \varDelta_{f} \) is a constant given in \eqref{eq: Delta_f}, being independent of \( \alpha \), \( m \), \( \nU \), \( \nL \), or \( \delta \).
\end{theorem}

Notably, \thmref{thm: importance weighting} imposes complicated sample size conditions on \( \nU \) and \( \nL \). These conditions arise from technical requirements needed to establish the high-probability bounds in the proof of the theorem. To interpret these requirements in a more concrete setting, we present the following corollary, which considers the scenario where the unlabeled sample size \( \nU \) scales polynomially with the labeled sample size \( \nL \).
\begin{corollary}
    \label{col: polynomial}
    Assume that the conditions of \thmref{thm: importance weighting} hold with \( M = \infty \), \( \iota \in [1/2, \tau] \), \( r \in [1/2, \tau] \), and \( G_{0}, \sigma_{0} > 0 \). Furthermore, suppose that there exists \( \beta \ge 1 \) such that \( \nU = \nL^{\beta} \). For arbitrarily small \( \varepsilon > 0 \), fix a sufficiently large moment parameter \( m \) such that
    \[
        2 \beta \nu = 2 \beta \LE( \frac{1}{m} \cdot \frac{2 \iota}{2 \iota + 1} \RI) < \varepsilon.
    \]
    In addition to \eqref{eq: sample size condition 1}, assume that the labeled sample size \( \nL \) is sufficiently large such that
    \begin{align*}
        \nL
        \ge \max \Bigg\{ & \LE( \frac{2 \kappa^{2}}{\min \family{ \norm{ \LT }, 1 }} \LE( (1 - \alpha)^{-3/2} \varDelta_{\phi} + \varXi_{m}^{1/2} \RI) \log \frac{36}{\delta} \RI)^{\LE. 1 \MID/ \LE( \beta \frac{\iota}{2 \iota + 1} \cdot \min \family{ 1, \frac{2}{2 r - 1} } - s - \varepsilon \RI) \RI.}, \\
                         & \LE( 14 \kappa^{2} \LE( \log \frac{24 \kappa^{2} (\norm{ \LT } + 2)}{\norm{ \LT } \delta} + 1 + \frac{1}{2 r} \RI) \RI)^{\LE. 1 \MID/ \LE( \frac{1}{2 r + 1} - s - \varepsilon \RI) \RI.} \Bigg\}.
    \end{align*}
    Then, for any \( \delta \in (0, 1) \), with probability at least \( 1 - \delta \), the following statements hold.
    \begin{enumerate}
        \item When \( 1 \le \beta < 1 + 1 / (2 \iota) \), set the regularization parameter as \( \lambda = \nL^{-s} \) with
              \[
                  s = \LE\{
                  \begin{alignedat}{2}
                       & \beta \frac{\iota}{2 \iota + 1} - \varepsilon,
                       &                                                                                           & \quad \frac{1}{2} \le r < \frac{1}{2} + \LE( \frac{1 + 1 / (2 \iota)}{\beta} - 1 \RI);                                                        \\
                       & \frac{1}{2 r + 1} - \varepsilon,
                       &                                                                                           & \quad \frac{1}{2} + \LE( \frac{1 + 1 / (2 \iota)}{\beta} - 1 \RI) \le r \le \frac{1}{2} + \LE( \frac{1 + 1 / (2 \iota)}{\beta} - 1 \RI)^{-1}; \\
                       & \beta \frac{\iota}{2 \iota + 1} \cdot \min \family{ 1, \frac{2}{2 r - 1} } - \varepsilon,
                       &                                                                                           & \quad r > \frac{1}{2} + \LE( \frac{1 + 1 / (2 \iota)}{\beta} - 1 \RI)^{-1}.
                  \end{alignedat}
                  \RI.
              \]
              The error bounds in \thmref{thm: importance weighting} then simplify to
              \[
                  \norm{ \frho - \hflam }_{\rhoTX} = O \LE( \nL^{-r s} \log \frac{36}{\delta} \RI),
                  \quad \norm{ \frho - \hflam }_{\calH} = O \LE( \nL^{-\LE( r - \frac{1}{2} \RI) s} \log \frac{36}{\delta} \RI).
              \]

        \item For the case \( \beta \ge 1 + 1 / (2 \iota) \), set
              \[
                  s = \frac{1}{2 r + 1} - \varepsilon,
                  \quad \forall r \ge \frac{1}{2}.
              \]
              We then obtain the following rates:
              \[
                  \norm{ \frho - \hflam }_{\rhoTX} = O \LE( \nL^{-\frac{r}{2 r + 1} + r \varepsilon} \log \frac{36}{\delta} \RI),
                  \quad \norm{ \frho - \hflam }_{\calH} = O \LE( \nL^{-\frac{r - 1/2}{2 r + 1} + \LE( r - \frac{1}{2} \RI) \varepsilon} \log \frac{36}{\delta} \RI).
              \]
    \end{enumerate}
\end{corollary}
We remark that in many real-world scenarios, unlabeled data from both the source and target distributions are cheap and abundant. For instance, in natural language processing, vast amounts of unlabeled text are readily available on the web, whereas high-quality labeled datasets require significant manual annotation effort. It is therefore reasonable to consider regimes where the unlabeled sample size \( \nU \) is much larger than the labeled sample size \( \nL \).

According to \citet{Caponnetto2007OptimalRR}, for a source condition parameter \( r \ge 1/2 \), the minimax optimal convergence rate (in the \( \calL^{2} \)-norm) for any learning algorithm is \( O(\nL^{-r / (2 r + 1)}) \). If the true density ratio were known a priori, this near-optimal rate can be attained for all \( r \ge 1 / 2 \) \citep{Gogolashvili2023WhenIW, Fan2025SpectralAU, Liu2025SpectralAM}. In our setting, however, the density ratio must be estimated from data. Consequently, the same near-optimal range \( r \ge 1 / 2 \) is achievable when the unlabeled sample size \( \nU \) scales polynomially with the labeled sample size \( \nL \), i.e., when \( \nU = \nL^{\beta} \) with \( \beta \ge 1 + 1 / (2 \iota) \); this condition is reasonable, as unlabeled data is typically much easier to acquire than labeled data in practical settings. For the case \( 1 \le \beta < 1 + 1 / (2 \iota) \), the near-optimal range narrows. 
We note that estimating an unbounded density ratio is a more challenging problem than covariate shift adaptation, requiring a substantially larger unlabeled sample to guarantee the same convergence rate.

We conclude this section by noting that our theoretical analysis relies on several integral operators defined with respect to different distributions. For clarity, we compile the definitions of these population and empirical operators in \tabref{tab: L} for reference.

\begin{table}
    \centering
    \small
    \begin{tabular}{ll|ll}
        \noalign{\hrule height 1.0pt}
         & Definition
         &
         & Definition                                                                                                 \\
        \noalign{\hrule height 0.5pt}
        \( \LS\)
         & \( \displaystyle \LS \, f = \int_{\calX} f(x) \, K(\cdot, x) \, \dd \rhoSX(x) \)
         & \( \hLS \)
         & \( \displaystyle \hLS \, f = \frac{1}{\nU} \sum_{j=1}^{\nU} f(x_{j}) \, K(\cdot, x_{j}) \rule{0pt}{2em} \) \\[1.5em]
        \( \LT \)
         & \( \displaystyle \LT \, f = \int_{\calX} f(x) \, K(\cdot, x) \, \dd \rhoTX(x) \)
         & \( \hLT \)
         & \( \displaystyle \hLT \, f = \frac{1}{\nU} \sum_{j=1}^{\nU} f(x'_{j}) \, K(\cdot, x'_{j}) \)               \\[1.5em]
        \( \LR \)
         & \( \displaystyle \LR \, f = (1 - \alpha) \LS \, f + \alpha \LT \, f \)
         & \( \hLR \)
         & \( \displaystyle \hLR \, f = (1 - \alpha) \hLS \, f + \alpha \hLT \, f \)                                  \\[0.5em]
        \( \LW \)
         & \( \displaystyle \LW \, f = \int_{\calX} \htheta(x) f(x) \, K(\cdot, x) \, \dd \rhoSX(x) \)
         & \( \hLW \)
         & \( \displaystyle \hLW \, f = \frac{1}{\nL} \sum_{i=1}^{\nL} \htheta(x_{i}) f(x_{i}) \, K(\cdot, x_{i}) \)  \\[1em]
        \noalign{\hrule height 1.0pt}
    \end{tabular}
    \caption{Definitions of integral operators and their empirical counterparts. The unlabeled samples \( \family{ x_{j} }_{j=1}^{\nU} \) and \( \family{ x'_{j} }_{j=1}^{\nU} \) are drawn from \( \rhoSX \) and \( \rhoTX \), respectively, while the labeled sample \( \{ x_i \}_{i=1}^{\nL} \) is drawn from \( \rhoSX \). The density ratio estimator \( \htheta = \hthetamuD \) is obtained via \eqref{eq: hthetamuD} with the optimal regularization parameter \( \mu \) and truncation threshold \( D \) specified in \thmref{thm: density ratio}.}
    \label{tab: L}
\end{table}

%% file: text/3_discussion.tex
\section{Related Works and Discussions}
\label{sec: discussion}

Covariate shift poses a fundamental challenge in machine learning. It occurs when input distributions differ between source and target data while the conditional distribution remains unchanged. To address this, the importance weighting method introduced by \citet{Shimodaira2000ImprovingPI} reweights source samples using the density ratio of target to source distributions. Although theoretically justified, this approach's efficacy relies critically on accurate estimation of the density ratio. Such estimation becomes especially difficult when the ratio is unbounded, leading to high variance in the importance-weighted estimator.

\subsection{Density Ratio Estimation}

Classical density ratio estimation methods typically rely on boundedness assumptions, either explicitly in their theoretical analysis or implicitly through regularization schemes. This category includes methods that match moment statistics between source and target distributions, such as kernel mean matching \citep{Huang2006CorrectingSS, Gretton2008CovariateSK, Gizewski2022RegularizationUD, Myleiko2025RecoveringRD, Gizewski2026ImpactSK}, as well as approaches that directly fit the density ratio by minimizing statistical divergences, including KLIEP (which minimizes the Kullback-Leibler divergence) \citep{Sugiyama2008DirectIE} and LSIF (which minimizes the squared loss, a special case of Bregman divergence) \citep{Kanamori2009LeastAD}. 

We note that another family of estimation methods, namely classifier-based approaches \citep{Qin1998InferencesCS, Gutmann2010NoiseCE, Menon2016LinkingLD}, is capable of estimating unbounded density ratios. These methods assign labels \( z = -1 \) and \( z = 1 \) to samples \( \family{ x_{j} }_{j=1}^{\nU} \) from the source distribution \( \rhoSX \) and \( \family{ x'_{j} }_{j=1}^{\nU} \) from the target distribution \( \rhoTX \), respectively. A classifier is then trained to distinguish between the two sample sets, and its estimated posterior probability \( \widehat{\operatorname{P}}(z = 1 \mid x) \) is used to construct the density ratio estimator:
\[
    \htheta(x) = \frac{\widehat{\operatorname{P}}(z = 1 \mid x)}{1 - \widehat{\operatorname{P}}(z = 1 \mid x)}.
\]
Since \( \widehat{\operatorname{P}}(z = 1 \mid x) \) can approach 1, the resulting estimator \( \htheta \) is not inherently bounded. However, the theoretical analysis for these methods typically assumes the true density ratio follows a parametric form, such as an exponential model \citep{Kanamori2010TheoreticalAD, Nagumo2025RobustSE}, which constitutes a strong prior assumption.

Recently, there has been growing interest in handling unbounded density ratios. \citet{Feng2024DeepNQ} propose a method that truncates the unbounded density ratio and employs the LSIF method \citep{Kanamori2009LeastAD} on neural networks to estimate the truncated ratio. Their theoretical analysis guarantees convergence to the true density ratio when the truncation threshold grows with the unlabeled sample size. However, their analysis relies on the H\"{o}lder continuity of the truncated density ratio. Since the truncation threshold depends on the sample size, the corresponding H\"{o}lder index also becomes sample-dependent. \citet{Xu2025EstimatingUD} relax this assumption by introducing the local H\"{o}lder class. This class requires functions to satisfy classical H\"{o}lder smoothness on any compact set (via a coordinate transformation), with smoothness bounds controlled by the scale of that compact set. This framework enables the characterization of density ratios with both unbounded domains and unbounded ranges. Finally, \citet{Zheng2026ErrorAD} estimates the logarithm of the density ratio to handle unboundedness. Their analysis assumes the true ratio is strictly bounded away from zero. We remark that these works derive error bounds in expectation rather than high-probability bounds.

Our approach diverges from these methods by leveraging the relative density ratio framework of \citet{Yamada2013RelativeDE}, which naturally produces bounded estimates that can be subsequently transformed to recover unbounded density ratios. The key insight is that while the standard density ratio \( \theta \) may be unbounded, the relative density ratio \( \phi = \theta / (\alpha \theta + (1 - \alpha)) \) is always bounded in \( [0, 1 / \alpha] \), making it amenable to RKHS estimation. As the moment parameter \( m \) increases, our theoretical analysis yields a convergence rate approaching \( O(\nU^{-\iota / (2 \iota + 1)}) \), where \( \iota \ge 1/2 \) characterizes the smoothness of the relative density ratio \( \phi \). Existing studies on relative density ratios primarily focus on analyzing their convergence properties (e.g., \citealp{Alejandro2024OnlineNL}) or employing them as a robust alternative to standard counterparts in importance-weighted learning (e.g., \citealp{Liu2013ChangeDT}). To the best of our knowledge, no research has yet explored transforming such ratios back to standard ones for estimating unbounded density ratios.

\subsection{Covariate Shift Adaptation}

In the context of covariate shift adaptation with importance-weighted kernel methods, previous analyses primarily rely on idealized assumptions regarding density ratios. \citet{Gogolashvili2023WhenIW} establish convergence rates for kernel ridge regression under the assumption of known density ratios satisfying Bernstein-type moment conditions, achieving near-optimal rates for well-specified models. This analysis is extended by \citet{Fan2025SpectralAU} to general spectral algorithms. Further broadening this framework, \citet{Liu2025SpectralAM} incorporate misspecified settings where the regression function may lie outside the RKHS, though still assuming known density ratios. Parallel to these developments, \citet{Ma2023OptimallyTC} analyze kernel ridge regression with known density ratios that are either uniformly bounded or possess finite second moments, but require the restrictive assumption that the eigenfunctions of the integral operators are uniformly bounded---a condition often difficult to verify in practice. This work is subsequently generalized by \citet{Feng2023TowardsUA} to accommodate broader loss functions beyond squared loss. Departing from the paradigm of known ratios, \citet{Gizewski2022RegularizationUD} address density ratio estimation under the assumption that the ratio belongs to an RKHS (implying uniform boundedness) and incorporate spectral algorithms for the final regression function estimator. Their follow-up work \citep{Gizewski2026ImpactSK} introduces a source condition on the kernel to enable a finer-grained analysis, and plugs the estimated density ratio as importance weights into the regression. \citet{Myleiko2025RecoveringRD} investigate a misspecified setting, estimating density ratio outside the RKHS, but their analysis retains the assumption that the density ratio is uniformly bounded.
 
Our work extends this literature by providing a unified analysis that incorporates both density ratio estimation and importance-weighted spectral algorithms, with particular emphasis on handling unbounded ratios. In our analysis, to guarantee fast convergence towards the regression function, the density ratio estimation step requires a substantially larger sample size than the subsequent regression step. Specifically, we assume the relationship \( \nU = \nL^{\beta} \) for some \( \beta \ge 1 \), where \( \nU \) denotes the unlabeled sample size used for density ratio estimation and \( \nL \) is the labeled sample size for covariate shift adaptation. When \( \beta \ge 1 + 1 / (2 \iota) \), we obtain a near-optimal convergence rate for all \( r \ge 1 / 2 \), where the parameter \( r \) characterizes the smoothness of the regression function. In contrast, when \( 1 \le \beta < 1 + 1 / (2 \iota) \), the near-optimal rate is only attainable within a finite range of \( r \). This finding provides new insight into the inherent difficulty of the unbounded density ratio estimation problem.

\subsection{Future Work}

Finally, we conclude by discussing several limitations of our work, which also suggest fruitful directions for future research. First, our analysis assumes the true regression function resides within the RKHS, corresponding to the smoothness condition \( r \ge 1/2 \). The misspecified setting, where \( r < 1/2 \), remains an open question and warrants further investigation. Second, while our theoretical framework relies on foundational RKHS theory, the kernel learning literature provides more refined tools for characterizing RKHS capacity, such as the effective dimension \citep{Caponnetto2007OptimalRR} or the embedding index \citep{Fischer2020SobolevNL}. These tools capture the nuanced interplay between RKHSs, spaces of uniformly bounded functions, and \( \calL^{2} \) spaces. Incorporating such advanced capacity notions could potentially yield sharper convergence rates. Third, in our density ratio estimation procedure, the combined effects of regularization, truncation, and the nonlinear transformation from a relative density ratio to the standard ratio all introduce bias into the final estimates. This bias may hinder the overall convergence of the estimator. Developing a debiasing strategy represents a promising direction for future work.

%% file: text/4_proof.tex
\section{Proofs of Main Results}
\label{sec: proof}

This section presents the proofs of \thmref{thm: relative density ratio}, \thmref{thm: density ratio}, \thmref{thm: importance weighting}, and \colref{col: polynomial}. Our proof strategy comprises three main steps: first, we perform an error decomposition of the target quantity; second, we establish individual bounds for each decomposed component through auxiliary propositions; and third, we combine these bounds to complete the overall argument. Auxiliary lemmas supporting these propositions are deferred to \hyperref[sec: lemma]{Appendix}.

\subsection{Proof of \thmref{thm: relative density ratio} and \thmref{thm: density ratio}}
\label{sec: proof densiti ratio}

To bound the convergence rate of the density ratio estimator (as measured by \( \norm{ \theta - \hthetamuD }_{\rhoSX} \)), we first analyze the convergence rate of the relative density ratio estimator (as measured by \( \norm{ \phi - \hphimu }_{\rhoRX} \)). We define the following auxiliary function:
\[
    \phimu = \gmu(\LR) \, \LR \, \phi.
\]
Then, we decompose \( \norm{ \phi - \hphimu }_{\rhoRX} \) as
\begin{equation}
    \label{eq: (phi) decompse 1}
    \norm{ \phi - \hphimu }_{\rhoRX} \le \norm{ \phi - \phimu }_{\rhoRX} + \norm{ \phimu - \hphimu }_{\rhoRX}.
\end{equation}
The second term \( \norm{ \phimu - \hphimu }_{\rhoRX} \) in \eqref{eq: (phi) decompse 1} can be further decomposed as follows:
\begin{equation}
    \label{eq: (phi) decompse 2}
    \begin{aligned}
        \norm{ \phimu - \hphimu }_{\rhoRX}
         & = \norm{ \LR^{1/2} \, (\phimu - \hphimu) }_{\calH}
        \le \norm{ \LR^{1/2} \, \LRmu^{-1/2} } \cdot \norm{ \LRmu^{1/2} \, \hLRmu^{-1/2} } \cdot \norm{ \hLRmu^{1/2} \, (\phimu - \hphimu) }_{\calH} \\
         & \le \norm{ \LRmu^{1/2} \, \hLRmu^{-1/2} } \cdot \norm{ \hLRmu^{1/2} \, (\phimu - \hphimu) }_{\calH},
    \end{aligned}
\end{equation}
where \( \norm{ \cdot } \) denotes the operator norm \( \norm{ \cdot }_{\calH \to \calH} \) on the RKHS \( \calH \), and the regularized operators are defined as \( \LRmu = \LR + \mu I \) and \( \hLRmu = \hLR + \mu I \). Recalling that \( \hphimu = \gmu(\hLR) \, \hLT \, \ind{\calX} \), we further decompose the last factor in \eqref{eq: (phi) decompse 2}:
\begin{equation}
    \label{eq: (phi) decompse 3}
    \begin{aligned}
        \norm{ \hLRmu^{1/2} \, (\phimu - \hphimu) }_{\calH}
         & = \norm{ \hLRmu^{1/2} \, (\phimu - \gmu(\hLR) \, \hLT \, \ind{\calX}) }_{\calH}                                                                                  \\
         & = \norm{ \hLRmu^{1/2} \, \LE( (I - \gmu(\hLR) \, \hLR + \gmu(\hLR) \, \hLR) \, \phimu - \gmu(\hLR) \, \hLT \, \ind{\calX} \RI) }_{\calH}                         \\
         & \le \norm{ \hLRmu^{1/2} \, \gmu(\hLR) \, (\hLR \, \phimu - \hLT \, \ind{\calX}) }_{\calH} + \norm{ \hLRmu^{1/2} \, (I - \gmu(\hLR) \, \hLR) \, \phimu }_{\calH}.
    \end{aligned}
\end{equation}
By combining \eqref{eq: (phi) decompse 1}, \eqref{eq: (phi) decompse 2}, and \eqref{eq: (phi) decompse 3}, we obtain
\[
    \norm{ \phimu - \hphimu }_{\rhoRX} \le P_{1} + P_{2} (P_{3} + P_{4}),
\]
where
\begin{equation}
    \label{eq: decompose P}
    \begin{alignedat}{2}
        P_{1} & = \norm{ \phi - \phimu }_{\rhoRX},
              & \quad P_{2}                                                                              & = \norm{ \LRmu^{1/2} \, \hLRmu^{-1/2} },                               \\
        P_{3} & = \norm{ \hLRmu^{1/2} \, \gmu(\hLR) \, (\hLR \, \phimu - \hLT \, \ind{\calX}) }_{\calH},
              & P_{4}                                                                                    & = \norm{ \hLRmu^{1/2} \, (I - \gmu(\hLR) \, \hLR) \, \phimu }_{\calH}.
    \end{alignedat}
\end{equation}

In the following, we bound \( P_{1} \), \( P_{2} \), \( P_{3} \), and \( P_{4} \) via separate propositions. We begin with \proref{pro: P1}, which establishes the bound for \( P_{1} \).
\begin{proposition}
    \label{pro: P1}
    Suppose that \aspref{asp: source phi} holds with \( \iota \in [1/2, \tau] \). Then the following bound holds:
    \[
        P_{1} = \norm{ \phi - \phimu }_{\rhoRX} \le F \norm{ \varpi }_{\rhoRX} \mu^{\iota}.
    \]
\end{proposition}
\begin{proof}
    By the definition of $ \phimu $, the quantity $ P_{1} $ can be expressed as
    \[
        P_{1} = \norm{ \phi - \phimu }_{\rhoRX}
        = \norm{ \LR^{1/2} \, (\phi - \phimu) }_{\calH}
        = \norm{ \LR^{1/2} \, (I - \gmu(\LR) \, \LR) \, \phi }_{\calH}.
    \]
    Under \aspref{asp: source phi}, we have \( \phimu = \LR^{\iota} \, \varpi \), which yields
    \begin{align*}
         & \eqspace \norm{ \LR^{1/2} \, (I - \gmu(\LR) \, \LR) \, \phi }_{\calH}                                                                                                                     \\
         & = \norm{ \LR^{1/2} \, (I - \gmu(\LR) \, \LR) \, \LR^{\iota} \, \varpi }_{\calH}
        \le \norm{ \LR^{\iota} \, (I - \LR \, \gmu(\LR)) } \cdot \norm{ \LR^{1/2} \, \varpi }_{\calH}                                                                                                \\
         & \overset{ \text{(a)} }{ \le } \sup_{t \in [0, \kappa^{2}]} t^{\iota} |1 - t \gmu(t)| \cdot \norm{ \varpi }_{\rhoRX} \overset{ \text{(b)} }{ \le } F \norm{ \varpi }_{\rhoRX} \mu^{\iota}.
    \end{align*}
    Step (a) uses the boundedness \( \norm{ \LR } \le \kappa^{2} \), which follows from \( \sup_{x \in \calX} K(x, x) = \kappa^{2} \); and step (b) applies the filter function property in \eqref{eq: 1 - t glam}. This completes the proof.
\end{proof}

\proref{pro: P2} provides an upper bound for \( P_{2} \) in \eqref{eq: decompose P}.
\begin{proposition}
    \label{pro: P2}
    For any \( \delta \in (0, 1) \), if the regularization parameter is chosen as \( \mu = \nU^{-\varsigma} \) with \( 0 < \varsigma < 1 \), and the sample size \( \nU \) is sufficiently large such that
    \begin{align*}
        \nU \ge 85 \kappa^{2} \LE( 1 + \log \frac{2 \kappa^{2} (\norm{ \LS } + \norm{ \LT } + 1)}{\min \family{ \norm{ \LS }, \norm{ \LT } } \cdot \delta} + \frac{2}{1 - \varsigma} \RI)^{\frac{2}{1 - \varsigma}},
    \end{align*}
    where \( \alpha \in (0, 1) \) is the mixing coefficient, then with probability at least \( 1 - \delta \), we have
    \[
        P_{2} = \norm{ \LRmu^{1/2} \, \hLRmu^{-1/2} } \le \sqrt{2}.
    \]
\end{proposition}
\begin{proof}
    Before bounding \( P_{2} \), we first derive an upper bound for \( \norm{ \LRmu^{-1/2} \, (\LR - \hLR) \, \LRmu^{-1/2} } \). Define the integral operator \( \LS \) with respect to the source distribution \( \rhoSX \) and its empirical counterpart \( \hLS \) as
    \[
        \LS \, f = \int_{\calX} f(x) \, K(\cdot, x) \, \dd \rhoSX(x),
        \quad \hLS \, f = \frac{1}{\nU} \sum_{j=1}^{\nU} f(x_{j}) \, K(\cdot, x_{j}).
    \]
    The operators \( \LR \) and \( \hLR \) satisfy the decompositions
    \[
        \LR = (1 - \alpha) \LS + \alpha \LT,
        \quad \hLR = (1 - \alpha) \hLS + \alpha \hLT,
    \]
    which implies
    \[
        \norm{ \LRmu^{-1/2} \, (\LR - \hLR) \, \LRmu^{-1/2} }
        \le (1 - \alpha) \norm{ \LRmu^{-1/2} \, (\LS - \hLS) \, \LRmu^{-1/2} } + \alpha \norm{ \LRmu^{-1/2} \, (\LT - \hLT) \, \LRmu^{-1/2} }.
    \]

    In the following, we apply \lemref{lem: concentration operator} to bound these two terms. Take the first term \( \norm{ \LRmu^{-1/2} \, (\LS - \hLS) \, \LRmu^{-1/2} } \) as an example:
    \begin{enumerate}
        \item Define the random variable \( A^{\text{(1)}}(x) = \LRmu^{-1/2} \, (\LS - \Kx \otimes \Kx) \, \LRmu^{-1/2} \) for \( x \sim \rhoSX \), with \( A^{\text{(1)}}_{j} = A^{\text{(1)}}(x_{j}) \), where the operator \( \Kx \otimes \Kx \) is defined by
              \begin{equation}
                  \label{eq: Kx otimes Kx}
                  \Kx \otimes \Kx \, f = f(x) \, K(\cdot, x).
              \end{equation}
              This definition yields the identity
              \[
                  \LRmu^{-1/2} \, (\LS - \hLS) \, \LRmu^{-1/2} = \frac{1}{\nU} \sum_{j=1}^{\nU} A^{\text{(1)}}_{j}.
              \]

        \item Since \( \norm{ \Kx \otimes \Kx } \le \kappa^{2} \) and \( \norm{ \LRmu^{-1/2} } \le \mu^{-1/2} \), we obtain
              \[
                  \norm{ \LRmu^{-1/2} \, \Kx \otimes \Kx \, \LRmu^{-1/2} } \le \norm{ \LRmu^{-1/2} } \cdot \norm{ \Kx \otimes \Kx } \cdot \norm{ \LRmu^{-1/2} } \le \kappa^{2} \mu^{-1}.
              \]
              This bound holds uniformly for all \( x \in \calX \), which implies
              \begin{equation}
                  \label{eq: (P2) A^(1)}
                  \norm{ A^{\text{(1)}}(x) } \le 2 \kappa^{2} \mu^{-1},
                  \quad \forall x \in \calX.
              \end{equation}

        \item For two positive operators \( A \) and \( B \), we write \( A \preceq B \) if \( B - A \) is positive. Then, the following bound holds:
              \begin{equation}
                  \label{eq: (P2) E (A^(1))^2}
                  \begin{aligned}
                       & \eqspace \EE[x \sim \rhoSX]{ \LE( A^{\text{(1)}}(x) \RI)^{2} }                                       \\
                       & \preceq \EE[x \sim \rhoSX]{ \LE( \LRmu^{-1/2} \, \Kx \otimes \Kx \, \LRmu^{-1/2} \RI)^{2} }
                      \preceq \kappa^{2} \mu^{-1} \cdot \EE[x \sim \rhoSX]{ \LRmu^{-1/2} \, \Kx \otimes \Kx \, \LRmu^{-1/2} } \\
                       & = \kappa^{2} \mu^{-1} \cdot \LRmu^{-1/2} \, \LS \, \LRmu^{-1/2}.
                  \end{aligned}
              \end{equation}
              Since \( (1 - \alpha) \LS \preceq \LRmu \) by the definition \( \LRmu = (1 - \alpha) \LS + \alpha \LT + \mu I \), it follows that for any \( f \in \calH \),
              \[
                  \inner{ (1 - \alpha) \LS \, (\LRmu^{-1/2} \, f), \LRmu^{-1/2} \, f }_{\calH}
                  \le \inner{ \LRmu \, (\LRmu^{-1/2} \, f), \LRmu^{-1/2} \, f }_{\calH}.
              \]
              Rewriting the inner products, we obtain
              \[
                  \inner{ (1 - \alpha) \LRmu^{-1/2} \, \LS \, \LRmu^{-1/2} \, f, f }_{\calH}
                  \le \inner{ \LRmu^{-1/2} \, \LRmu \, \LRmu^{-1/2} \, f, f }_{\calH}
                  = \inner{ f, f }_{\calH},
              \]
              which implies \( (1 - \alpha) \LRmu^{-1/2} \, \LS \, \LRmu^{-1/2} \preceq I \). Therefore, we have the upper bound
              \begin{equation}
                  \label{eq: (P2) LRmu^-1/2 LS LRmu^-1/2 upper}
                  \norm{ \LRmu^{-1/2} \, \LS \, \LRmu^{-1/2} } \le \frac{1}{1 - \alpha}.
              \end{equation}

        \item For the lower bound, observe that
              \[
                  \LRmu
                  = (1 - \alpha) \LS + \alpha \LT + \mu I
                  \preceq (1 - \alpha) \LS + (\alpha \norm{ \LT } + \mu) I,
              \]
              where the inequality uses \( \LT \preceq \norm{ \LT } \cdot I \). This leads to
              \[
                  \LRmu^{-1} \succeq ((1 - \alpha) \LS + (\alpha \norm{ \LT } + \mu) I)^{-1}.
              \]
              For any \( f \in \calH \), it follows that
              \[
                  \inner{ \LRmu^{-1} \, (\LS^{1/2} \, f), \LS^{1/2} \, f }_{\calH}
                  \ge \inner{ ((1 - \alpha) \LS + (\alpha \norm{ \LT } + \mu) I)^{-1} \, (\LS^{1/2} \, f), \LS^{1/2} \, f }_{\calH}.
              \]
              Expressing this in terms of norms, we have
              \begin{align*}
                   & \eqspace \norm{ \LRmu^{-1/2} \, \LS^{1/2} \, f }_{\calH}^{2}                                                                                                                   \\
                   & = \inner{ \LRmu^{-1/2} \, \LS^{1/2} \, f, \LRmu^{-1/2} \, \LS^{1/2} \, f }_{\calH}                                                                                             \\
                   & \ge \inner{ ((1 - \alpha) \LS + (\alpha \norm{ \LT } + \mu) I)^{-1/2} \, \LS^{1/2} \, f, ((1 - \alpha) \LS + (\alpha \norm{ \LT } + \mu) I)^{-1/2} \, \LS^{1/2} \, f }_{\calH} \\
                   & = \norm{ (\alpha \LS + ((1 - \alpha) \norm{ \LT } + \mu) I)^{-1/2} \, \LS^{1/2} \, f }_{\calH}^{2}.
              \end{align*}
              Since this inequality holds for all \( f \in \calH \), we conclude that
              \begin{equation}
                  \label{eq: (P2) LRmu^-1/2 LS LRmu^-1/2 lower}
                  \begin{aligned}
                       & \eqspace \norm{ \LRmu^{-1/2} \, \LS \, \LRmu^{-1/2} }                                \\
                       & = \norm{ \LRmu^{-1/2} \, \LS^{1/2} }^{2}
                      \ge \norm{ ((1 - \alpha) \LS + (\alpha \norm{ \LT } + \mu) I)^{-1/2} \, \LS^{1/2} }^{2} \\
                       & = \norm{ ((1 - \alpha) \LS + (\alpha \norm{ \LT } + \mu) I)^{-1} \, \LS }
                      = \frac{\norm{ \LS }}{(1 - \alpha) \norm{ \LS } + \alpha \norm{ \LT } + \mu}.
                  \end{aligned}
              \end{equation}

        \item Finally, we have \( (1 - \alpha) \LS \preceq \LR \), which implies
              \[
                  (1 - \alpha) \LRmu^{-1/2} \, \LS \, \LRmu^{-1/2} \preceq \LRmu^{-1/2} \, \LR \, \LRmu^{-1/2}.
              \]
              Since the trace satisfies \( \Tr{ \LR } \le \kappa^{2} \) (see, e.g., \citealp{Caponnetto2007OptimalRR}), it follows that
              \begin{equation}
                  \label{eq: (P2) LRmu^-1/2 LS LRmu^-1/2 trace}
                  \begin{aligned}
                      \Tr{ \LRmu^{-1/2} \, \LS \, \LRmu^{-1/2} }
                       & \le \frac{1}{1 - \alpha} \Tr{ \LRmu^{-1/2} \, \LR \, \LRmu^{-1/2} }
                      = \frac{1}{1 - \alpha} \Tr{ \LRmu^{-1} \, \LR }                        \\
                       & \le \frac{1}{1- \alpha} \mu^{-1} \Tr{ \LR }
                      \le \frac{1}{1 - \alpha} \kappa^{2} \mu^{-1}.
                  \end{aligned}
              \end{equation}
    \end{enumerate}
    Combining the bounds from \eqref{eq: (P2) A^(1)}, \eqref{eq: (P2) E (A^(1))^2}, \eqref{eq: (P2) LRmu^-1/2 LS LRmu^-1/2 upper}, \eqref{eq: (P2) LRmu^-1/2 LS LRmu^-1/2 lower}, and \eqref{eq: (P2) LRmu^-1/2 LS LRmu^-1/2 trace}, and applying \lemref{lem: concentration operator}, we obtain with probability at least \( 1 - \delta / 2 \) that
    \[
        \norm{ \LRmu^{-1/2} \, (\LS - \hLS) \, \LRmu^{-1/2} }
        = \norm{ \frac{1}{\nU} \sum_{j=1}^{\nU} A^{\text{(1)}}_{j} }
        \le \frac{4 \kappa^{2} \mu^{-1}}{3 \nU} h^{\text{(1)}} + \sqrt{\frac{(2 / (1 - \alpha)) \kappa^{2} \mu^{-1}}{\nU} h^{\text{(1)}}},
    \]
    where \( h^{\text{(1)}} \) is defined as
    \[
        h^{\text{(1)}} = \log \LE( \frac{4}{1 - \alpha} \kappa^{2} \mu^{-1} \MID/ \LE( \frac{\norm{ \LS }}{(1 - \alpha) \norm{ \LS } + \alpha \norm{ \LT } + \mu} \delta \RI) \RI).
    \]
    Now, set \( \mu = \nU^{-\varsigma} \) for some \( 0 < \varsigma < 1 \). Using the inequalities
    \[
        \log \frac{2}{1 - \alpha} \le \frac{2}{1 - \alpha},
        \quad \log \nU \le \frac{2}{1 - \varsigma} \nU^{\frac{1 - \varsigma}{2}},
    \]
    we bound \( h^{\text{(1)}} \) as
    \[
        h^{\text{(1)}} \le \frac{2}{1 - \alpha} \LE( 1 + \log \frac{2 \kappa^{2} (\norm{ \LS } + \norm{ \LT } + 1)}{\norm{ \LS } \delta} + \frac{2}{1 - \varsigma} \RI) \nU^{\frac{1 - \varsigma}{2}}.
    \]
    By analogous reasoning, we also have with probability at least \( 1 - \delta / 2 \) that
    \[
        \norm{ \LRmu^{-1/2} \, (\LT - \hLT) \, \LRmu^{-1/2} }
        \le \frac{4 \kappa^{2} \mu^{-1}}{3 \nU} h^{\text{(2)}} + \sqrt{\frac{(2 / \alpha) \kappa^{2} \mu^{-1}}{\nU} h^{\text{(2)}}},
    \]
    where
    \[
        h^{\text{(2)}} \le \frac{2}{\alpha} \LE( 1 + \log \frac{2 \kappa^{2} (\norm{ \LS } + \norm{ \LT } + 1)}{\norm{ \LT } \delta} + \frac{2}{1 - \varsigma} \RI) \nU^{\frac{1 - \varsigma}{2}}.
    \]
    By letting
    \[
        \nU \ge 85 \kappa^{2} \LE( 1 + \log \frac{2 \kappa^{2} (\norm{ \LS } + \norm{ \LT } + 1)}{\min \family{ \norm{ \LS }, \norm{ \LT } } \cdot \delta} + \frac{2}{1 - \varsigma} \RI)^{\frac{2}{1 - \varsigma}},
    \]
    we ensure that with probability at least \( 1 - \delta \),
    \begin{align*}
         & \eqspace \norm{ \LRmu^{-1/2} \, (\LR - \hLR) \, \LRmu^{-1/2} }                                                                                                                                                                                                  \\
         & \le (1 - \alpha) \norm{ \LRmu^{-1/2} \, (\LS - \hLS) \, \LRmu^{-1/2} } + \alpha \norm{ \LRmu^{-1/2} \, (\LT - \hLT) \, \LRmu^{-1/2} }                                                                                                                           \\
         & \le (1 - \alpha) \LE( \frac{2}{1 - \alpha} \cdot \frac{16}{1000} + \frac{2}{1 - \alpha} \cdot \frac{109}{1000} \RI) + \alpha \LE( \frac{2}{1 - \alpha} \cdot \frac{16}{1000} + \frac{2}{\alpha} \cdot \frac{109}{1000} \RI) \\
         & = \frac{1}{2}.
    \end{align*}
    Using the Neumann series expansion, we obtain
    \begin{align*}
        P_{2}^{2} & = \norm{ \LRmu^{1/2} \, \hLRmu^{-1/2} }^{2}
        = \norm{ \LRmu^{1/2} \, \hLRmu^{-1} \, \LRmu^{1/2} }
        = \norm{ \LE( \LRmu^{-1/2} \, \hLRmu \, \LRmu^{-1/2} \RI)^{-1} }                        \\
                  & = \norm{ \LE( I - \LRmu^{-1/2} \, (\LR - \hLR) \, \LRmu^{-1/2} \RI)^{-1} }
        \le \sum_{\ell=0}^{\infty} \norm{ \LRmu^{-1/2} \, (\LR - \hLR) \, \LRmu^{-1/2} }^{\ell} \\
                  & \le 2,
    \end{align*}
    It follows that \( P_{2} \le \sqrt{2} \) with probability at least \( 1 - \delta \).
\end{proof}

We can bound \( P_{3} \) in \eqref{eq: decompose P} using \proref{pro: P3}.
\begin{proposition}
    \label{pro: P3}
    Suppose that \aspref{asp: source phi} holds with \( \iota \in [1/2, \tau] \), and condition on the event that the bound for \( P_{2} \) given in \proref{pro: P2} holds. For any \( \delta \in (0, 1) \), if, in addition to the conditions in \proref{pro: P2}, the regularization parameter \( \mu = \nU^{-\varsigma} \) satisfies \( \varsigma \le 1 / (2 \iota + 1) \), then
    \begin{align*}
        P_{3}
        = \norm{ \hLRmu^{1/2} \, \gmu(\hLR) \, (\hLR \, \phimu - \hLT \, \ind{\calX}) }_{\calH}
        \le 20 \sqrt{2} E \LE( (\kappa^{2 \iota + 1} E + F) \norm{ \varpi }_{\rhoRX} + \kappa \RI) \mu^{\iota} \log \frac{6}{\delta}
    \end{align*}
    with probability at least \( 1 - \delta \).
\end{proposition}
\begin{proof}
    We begin by decomposing \( P_{3} \) as follows:
    \begin{equation}
        \label{eq: (P3) decompose 1}
        \begin{aligned}
            P_{3}
             & = \norm{ \hLRmu^{1/2} \, \gmu(\hLR) \, (\hLR \, \phimu - \hLT \, \ind{\calX}) }_{\calH}                                                                                            \\
             & \le \norm{ \hLRmu^{1/2} \, \gmu(\hLR) \, \hLRmu^{1/2} } \cdot \norm{ \hLRmu^{-1/2} \, \LRmu^{1/2} } \cdot \norm{ \LRmu^{-1/2} \, (\hLR \, \phimu - \hLT \, \ind{\calX}) }_{\calH}.
        \end{aligned}
    \end{equation}
    By the filter function property in \eqref{eq: glam}, we bound the first factor:
    \[
        \norm{ \hLRmu^{1/2} \, \gmu(\hLR) \, \hLRmu^{1/2} }
        = \norm{ \hLRmu \, \gmu(\hLR) }
        \le \sup_{t \in [0, \kappa^{2}]} (t + \mu) \gmu(t)
        \le E + \mu \cdot E \mu^{-1} = 2 E.
    \]
    The second factor is bounded by \( P_{2} \) from \proref{pro: P2}:
    \[
        \norm{ \hLRmu^{-1/2} \, \LRmu^{1/2} } = P_{2} \le \sqrt{2}.
    \]
    It remains to bound the third factor, \( \norm{ \LRmu^{-1/2} \, (\hLR \, \phimu - \hLT \, \ind{\calX}) }_{\calH} \). We decompose it further:
    \begin{equation}
        \label{eq: (P3) decompose 2}
        \begin{aligned}
             & \eqspace \norm{ \LRmu^{-1/2} \, (\hLR \, \phimu - \hLT \, \ind{\calX}) }_{\calH}                                                                                                                    \\
             & \le \norm{ \LRmu^{-1/2} \, (\hLR - \LR) \, \phimu }_{\calH} + \norm{ \LRmu^{-1/2} \, (\LR \, \phimu - \LT \, \ind{\calX}) }_{\calH} + \norm{ \LRmu^{-1/2} \, (\LT - \hLT) \, \ind{\calX} }_{\calH}.
        \end{aligned}
    \end{equation}

    For the first term in \eqref{eq: (P3) decompose 2}, we split it using the definitions of \( \hLR \) and \( \LR \):
    \[
        \norm{ \LRmu^{-1/2} \, (\hLR - \LR) \, \phimu }_{\calH}
        \le (1 - \alpha) \norm{ \LRmu^{-1/2} \, (\hLS - \LS) \, \phimu }_{\calH} + \alpha \norm{ \LRmu^{-1/2} \, (\hLT - \LT) \, \phimu }_{\calH}.
    \]
    We bound these two norms separately using \lemref{lem: concentration vector}. For \( \norm{ \LRmu^{-1/2} \, (\hLS - \LS) \, \phimu }_{\calH} \), define the random variable \( \xi^{\text{(1)}}(x) = \LRmu^{-1/2} \, \Kx \otimes \Kx \, \phimu \) for \( x \sim \rhoSX \), with \( \xi^{\text{(1)}}_{j} = \xi^{\text{(1)}}(x_{j}) \). Then,
    \[
        \norm{ \LRmu^{-1/2} \, (\hLS - \LS) \, \phimu }_{\calH} = \norm{ \frac{1}{\nU} \sum_{j=1}^{\nU} \xi^{\text{(1)}}_{j} - \EE[x \sim \rhoSX]{ \xi^{\text{(1)}}(x) } }.
    \]
    Note that \( \xi^{\text{(1)}}(x) = \phimu(x) \cdot \LRmu^{-1/2} \, K(\cdot, x) \). To bound \( |\phimu(x)| \), we proceed as follows:
    \begin{align*}
        |\phimu(x)|
         & = \LE| \inner{ \phimu, K(\cdot, x) }_{\calH} \RI|
        \le \kappa \cdot \norm{ \phimu }_{\calH}
        = \kappa \cdot \norm{ \gmu(\LR) \, \LR \, \phi }_{\calH}
        = \kappa \cdot \norm{ \gmu(\LR) \, \LR^{\iota + 1} \, \varpi }_{\calH}                                                            \\
         & \le \kappa \cdot \norm{ \gmu(\LR) \, \LR } \cdot \norm{ \LR^{\iota - \frac{1}{2}} } \cdot \norm{ \LR^{1/2} \, \varpi }_{\calH}
        \le \kappa^{2 \iota} E \norm{ \varpi }_{\rhoRX}.
    \end{align*}
    Consequently, for any \( x \in \calX \), we have
    \[
        \norm{ \xi^{\text{(1)}}(x) }_{\calH} \le \kappa^{2 \iota} E \norm{ \varpi }_{\rhoRX} \norm{ \LRmu^{-1/2} \, K(\cdot, x) }_{\calH} \le \kappa^{2 \iota + 1} E \norm{ \varpi }_{\rhoRX} \mu^{-1/2}.
    \]
    By \lemref{lem: concentration vector}, with probability at least \( 1 - \delta / 3 \),
    \[
        \norm{ \LRmu^{-1/2} \, (\hLS - \LS) \, \phimu }_{\calH}
        \le 10 \kappa^{2 \iota + 1} E \norm{ \varpi }_{\rhoRX} \cdot \frac{\mu^{-1/2}}{\sqrt{\nU}} \log \frac{6}{\delta}.
    \]
    Analogously, defining \( \xi^{\text{(2)}}(x') = \LRmu^{-1/2} \, \Kxp \otimes \Kxp \, \phimu \) for \( x' \sim \rhoTX \), we obtain a similar bound for the term \( \norm{ \LRmu^{-1/2} \, (\hLT - \LT) \, \phimu }_{\calH} \). By a union bound, with probability at least \( 1 - 2\delta / 3 \),
    \begin{equation}
        \label{eq: (P3) component 1}
        \norm{ \LRmu^{-1/2} \, (\hLR - \LR) \, \phimu }_{\calH}
        \le 10 \kappa^{2 \iota + 1} E \norm{ \varpi }_{\rhoRX} \cdot \frac{\mu^{-1/2}}{\sqrt{\nU}} \log \frac{6}{\delta}.
    \end{equation}

    For the second term in \eqref{eq: (P3) decompose 2}, we use the identity \( \LT \, \ind{\calX} = \LR \, \phi \):
    \begin{equation}
        \label{eq: (P3) component 2}
        \begin{aligned}
            \norm{ \LRmu^{-1/2} \, (\LR \, \phimu - \LT \, \ind{\calX}) }_{\calH}
             & = \norm{ \LRmu^{-1/2} \, \LR \, (\phimu - \phi) }_{\calH}
            \le \norm{ \LRmu^{-1/2} \, \LR^{1/2} } \cdot \norm{ \LR^{1/2} \, (\phimu - \phi) }_{\calH} \\
             & \le \norm{ \LR^{1/2} \, (\phimu - \phi) }_{\calH}
            = \norm{ \phimu - \phi }_{\rhoRX}
            = P_{1}                                                                                    \\
             & \le F \norm{ \varpi }_{\rhoRX} \mu^{\iota},
        \end{aligned}
    \end{equation}
    where the last inequality follows from \proref{pro: P1}.

    For the third term in \eqref{eq: (P3) decompose 2}, define \( \xi^{\text{(3)}}(x') = \LRmu^{-1/2} \, \Kxp \otimes \Kxp \, \ind{\calX} = \LRmu^{-1/2} \, K(\cdot, x') \) for \( x' \sim \rhoTX \), with \( \xi^{\text{(3)}}_{j} = \xi^{\text{(3)}}(x'_{j}) \). Then \( \sup_{x' \in \calX} \norm{ \xi^{\text{(3)}}(x') }_{\calH} \le \kappa \mu^{-1/2} \). By \lemref{lem: concentration vector}, with probability at least \( 1 - \delta / 3 \),
    \begin{equation}
        \label{eq: (P3) component 3}
        \norm{ \LRmu^{-1/2} \, (\LT - \hLT) \, \ind{\calX} }_{\calH} = \norm{ \frac{1}{\nU} \sum_{j=1}^{\nU} \xi^{\text{(3)}}_{j} - \EE[x' \sim \rhoTX]{ \xi^{\text{(3)}}(x') } } \le 10 \kappa \cdot \frac{\mu^{-1/2}}{\sqrt{\nU}} \log \frac{6}{\delta}.
    \end{equation}

    With \( \mu = \nU^{-\varsigma} \) and \( \varsigma \le 1 / (2 \iota + 1) \), we have \( \mu^{-1/2} / \sqrt{\nU} \le \mu^{\iota} \). Combining \eqref{eq: (P3) component 1}, \eqref{eq: (P3) component 2}, and \eqref{eq: (P3) component 3} via \eqref{eq: (P3) decompose 2}, and then substituting into \eqref{eq: (P3) decompose 1}, we obtain
    \[
        P_{3} \le 2 \sqrt{2} E \cdot \norm{ \LRmu^{-1/2} \, (\hLR \, \phimu - \hLT \, \ind{\calX}) }_{\calH} \le 20 \sqrt{2} E \LE( (\kappa^{2 \iota + 1} E + F) \norm{ \varpi }_{\rhoRX} + \kappa \RI) \mu^{\iota} \log \frac{6}{\delta}.
    \]
    The union bound ensures this holds with probability at least \( 1 - \delta \).
\end{proof}

Finally, \proref{pro: P4} yields a bound for \( P_{4} \) in \eqref{eq: decompose P}.
\begin{proposition}
    \label{pro: P4}
    Suppose that \aspref{asp: source phi} holds with \( \iota \in [1/2, \tau] \), and condition on the event that the bound for \( P_{2} \) given in \proref{pro: P2} holds. For any \( \delta \in (0, 1) \), if the regularization parameter \( \mu = \nU^{-\varsigma} \) satisfies (in addition to the conditions in \proref{pro: P2})
    \[
        \varsigma \le
        \begin{cases}
            1/2,               & \iota \in ( 1, 3/2 ]; \\
            1 / (2 \iota - 1), & \iota > 3/2,
        \end{cases}
    \]
    then
    \[
        P_{4} = \norm{ \hLRmu^{1/2} \, (I - \gmu(\hLR) \, \hLR) \, \phimu }_{\calH}
        \le 2 E F \norm{ \varpi }_{\rhoRX} \LE( 10 \iota \kappa^{2 \iota - 1} \cdot \log \frac{4}{\delta} + 1 \RI) \mu^{\iota}
    \]
    with probability at least \( 1 - \delta \).
\end{proposition}
\begin{proof}
    Starting from the definition of \( \phimu \) and applying \aspref{asp: source phi}:
    \begin{align*}
        P_{4}
         & = \norm{ \hLRmu^{1/2} \, (I - \gmu(\hLR) \, \hLR) \, \phimu }_{\calH}
        = \norm{ \hLRmu^{1/2} \, (I - \gmu(\hLR) \, \hLR) \, \gmu(\LR) \, \LR \, \phi }_{\calH}                                                                           \\
         & = \norm{ \hLRmu^{1/2} \, (I - \gmu(\hLR) \, \hLR) \, \gmu(\LR) \, \LR^{\iota + 1} \, \varpi }_{\calH}                                                          \\
         & \le \norm{ \hLRmu^{1/2} \, (I - \gmu(\hLR) \, \hLR) \, \LR^{\iota - \frac{1}{2}} } \cdot \norm{ \gmu(\LR) \, \LR } \cdot \norm{ \LR^{1/2} \, \varpi }_{\calH}.
    \end{align*}
    The second and third factors are bounded respectively by
    \[
        \quad \norm{ \gmu(\LR) \, \LR } \le \sup_{t \in [0, \kappa^{2}]} t \gmu(t)
        \le E,
        \quad \norm{ \LR^{1/2} \, \varpi }_{\calH} = \norm{ \varpi }_{\rhoRX}.
    \]
    We further decompose the first factor:
    \begin{align*}
         & \eqspace \norm{ \hLRmu^{1/2} \, (I - \gmu(\hLR) \, \hLR) \, \LR^{\iota - \frac{1}{2}} }                                                                                                                       \\
         & = \norm{ \hLRmu^{1/2} \, (I - \gmu(\hLR) \, \hLR) \, (\LR^{\iota - \frac{1}{2}} - \hLR^{\iota - \frac{1}{2}} + \hLR^{\iota - \frac{1}{2}}) }                                                                  \\
         & \le \norm{ \hLRmu^{1/2} \, (I - \gmu(\hLR) \, \hLR) } \cdot \norm{ \LR^{\iota - \frac{1}{2}} - \hLR^{\iota - \frac{1}{2}} } + \norm{ \hLRmu^{1/2} \, (I - \gmu(\hLR) \, \hLR) \, \hLR^{\iota - \frac{1}{2}} } \\
         & \le 2 F \LE( \mu^{1/2} \cdot \norm{ \LR^{\iota - \frac{1}{2}} - \hLR^{\iota - \frac{1}{2}} } + \mu^{\iota} \RI),
    \end{align*}
    where we use the bounds
    \[
        \norm{ \hLRmu^{1/2} \, (I - \gmu(\hLR) \, \hLR) } \le 2 F \mu^{1/2},
        \quad \norm{ \hLRmu^{1/2} \, (I - \gmu(\hLR) \, \hLR) \, \hLR^{\iota - \frac{1}{2}} } \le 2 F \mu^{\iota}.
    \]
    By \lemref{lem: A^c - B^c}, we bound the operator difference by
    \begin{equation}
        \label{eq: operator difference LR - hLR}
        \norm{ \LR^{\iota - \frac{1}{2}} - \hLR^{\iota - \frac{1}{2}} }
        \le \begin{cases}
            \norm{ \LR - \hLR }^{\iota - \frac{1}{2}},              & \iota \in [1/2, 3/2]; \\
            (\iota - 1/2) \kappa^{2 \iota - 3} \norm{ \LR - \hLR }, & \iota > 3/2.
        \end{cases}
    \end{equation}
    We now bound \( \norm{ \LR - \hLR } \). By the definitions of \( \LR \) and \( \hLR \),
    \[
        \norm{ \LR - \hLR } \le (1 - \alpha) \norm{ \LS - \hLS } + \alpha \norm{ \LT - \hLT }.
    \]
    We estimate each term on the right-hand side by applying \lemref{lem: concentration vector}. Define the random variable \( \xi^{\text{(1)}}(x) = \Kx \otimes \Kx \) for \( x \sim \rhoSX \), with \( \xi^{\text{(1)}}_{j} = \xi^{\text{(1)}}(x_{j}) \). Then,
    \[
        \LS - \hLS = \EE[x \sim \rhoSX]{ \xi^{\text{(1)}}(x) } - \frac{1}{\nU} \sum_{j=1}^{\nU} \xi^{\text{(1)}}_{j}.
    \]
    Viewing \( \xi^{\text{(1)}} \) as an element in the Hilbert space \( \mathscr{H} \) of Hilbert-Schmidt operators, we have the uniform bound
    \[
        \norm{ \xi^{\text{(1)}}(x) }_{\mathscr{H}} \le \kappa^{2},
        \quad \forall x \in \calX.
    \]
    By \lemref{lem: concentration vector}, we obtain with probability at least \( 1 - \delta / 2 \) that
    \[
        \norm{ \LS - \hLS } \le \norm{ \LS - \hLS }_{\mathscr{H}} \le \frac{10 \kappa^{2}}{\sqrt{\nU}} \log \frac{4}{\delta}.
    \]
    For \( \norm{ \LT - \hLT } \), we obtain a similar bound with probability at least \( 1 - \delta / 2 \). By a union bound, with probability at least \( 1 - \delta \),
    \[
        \norm{ \LR - \hLR } \le \frac{10 \kappa^{2}}{\sqrt{\nU}} \log \frac{4}{\delta}.
    \]
    By \eqref{eq: operator difference LR - hLR}, this leads to
    \[
        \norm{ \LR^{\iota - \frac{1}{2}} - \hLR^{\iota - \frac{1}{2}} }
        \le 10 \iota \kappa^{2 \iota - 1} \nU^{-\frac{\min \family{ 2 \iota, 3 } - 1}{4}} \log \frac{4}{\delta}.
    \]
    Setting \( \mu = \nU^{-\varsigma} \) with \( \varsigma \le 1/2 \) for \( \iota \in ( 1, 3/2 ] \) and \( \varsigma \le 1 / (2 \iota - 1) \) for \( \iota > 3/2 \), we ensure
    \[
        \mu^{1/2} \nU^{-\frac{\min \family{ 2 \iota, 3 } - 1}{4}} \le \mu^{\iota}.
    \]
    Applying these estimates, we obtain
    \begin{align*}
        P_{4}
        \le 2 E F \norm{ \varpi }_{\rhoRX} \LE( \mu^{1/2} \cdot \norm{ \LR^{\iota - \frac{1}{2}} - \hLR^{\iota - \frac{1}{2}} } + \mu^{\iota} \RI)
        \le 2 E F \norm{ \varpi }_{\rhoRX} \LE( 10 \iota \kappa^{2 \iota - 1} \cdot \log \frac{4}{\delta} + 1 \RI) \mu^{\iota}
    \end{align*}
    with probability at least \( 1 - \delta \).
\end{proof}

Now we are ready to prove \thmref{thm: relative density ratio}.
\begin{proofof}{\thmref{thm: relative density ratio}}
    Set the regularization parameter as \( \mu = \nU^{-\varsigma} \) with \( \varsigma = 1 / (2 \iota + 1) \). For sufficiently large sample size \( \nU \) satisfying
    \begin{align*}
        \nU \ge 85 \kappa^{2} \LE( 1 + \log \frac{6 \kappa^{2} (\norm{ \LS } + \norm{ \LT } + 1)}{\min \family{ \norm{ \LS }, \norm{ \LT } } \cdot \delta} + \frac{2}{1 - \varsigma} \RI)^{\frac{2}{1 - \varsigma}},
    \end{align*}
    the conditions of \proref{pro: P1}, \proref{pro: P2}, \proref{pro: P3}, and \proref{pro: P4} are simultaneously satisfied. By a union bound, with probability at least \( 1 - \delta \) (assigning probability \( \delta/3 \) to each of \proref{pro: P2}, \proref{pro: P3}, and \proref{pro: P4}), we have
    \begin{align*}
        \norm{ \phi - \hphimu }_{\rhoRX}
         & \le P_{1} + P_{2} (P_{3} + P_{4})                                                                                                                   \\
         & \le F \norm{ \varpi }_{\rhoRX} \mu^{\iota}                                                                                                          \\
         & \eqspace + \sqrt{2} \cdot 20 \sqrt{2} E \LE( (\kappa^{2 \iota + 1} E + F) \norm{ \varpi }_{\rhoRX} + \kappa \RI) \mu^{\iota} \log \frac{18}{\delta} \\
         & \eqspace + \sqrt{2} \cdot 2 E F \norm{ \varpi }_{\rhoRX} \LE( 10 \iota \kappa^{2 \iota - 1} \cdot \log \frac{12}{\delta} + 1 \RI) \mu^{\iota}       \\
         & \le \varDelta_{\phi} \nU^{-\frac{\iota}{2 \iota + 1}} \log \frac{18}{\delta},
    \end{align*}
    where
    \begin{equation}
        \label{eq: Delta_phi}
        \varDelta_{\phi} = F \norm{ \varpi }_{\rhoRX} + 40 E \LE( (\kappa^{2 \iota + 1} E + F) \norm{ \varpi }_{\rhoRX} + \kappa \RI) + 2 \sqrt{2} E F \norm{ \varpi }_{\rhoRX} \LE( 10 \iota \kappa^{2 \iota - 1} + 1 \RI)
    \end{equation}
    is a constant independent of \( \alpha \), \( \nU \) or \( \delta \). This completes the proof.
\end{proofof}

The proof of \thmref{thm: density ratio} follows directly from the bound for \( \norm{ \phi - \hphimu }_{\rhoRX} \). Recall that the density ratio estimator \( \hthetamuD \) relates to the truncated relative density ratio \( \hphimuD \) via \eqref{eq: hphimuD} and \eqref{eq: hthetamuD}:
\[
    \hphimuD(x) = \min \left\{ \max \family{ \hphimu(x), 0}, \frac{D}{\alpha D + (1 - \alpha)} \right\},
    \quad \hthetamuD(x) = \frac{(1 - \alpha) \hphimuD(x)}{1 - \alpha \hphimuD(x)}.
\]
\begin{proofof}{\thmref{thm: density ratio}}

    To bound \( \norm{ \theta - \hthetamuD }_{\rhoSX} \), we decompose the input space \( \calX \) into two regions. Let \( \Omega_{D} = \family{ x \in \calX: \theta(x) \le D } \), then
    \begin{equation}
        \label{eq: theta - hthetabuD}
        \begin{aligned}
             & \eqspace \norm{ \theta - \hthetamuD }_{\rhoSX}^{2}                                                                                                                                                                                                                                 \\
             & = \int_{\calX} (\theta(x) - \hthetamuD(x))^{2} \, \dd \rhoSX(x)                                                                                                                                                                                                                    \\
             & = \int_{\Omega_{D}} \LE( \frac{(1 - \alpha) \phi(x)}{1 - \alpha \phi(x)} - \frac{(1 - \alpha) \hphimuD(x)}{1 - \alpha \hphimuD(x)} \RI)^{2} \, \dd \rhoSX(x) + \int_{\calX \setminus \Omega_{D}} \LE( 1 - \frac{\hthetamuD(x)}{\theta(x)} \RI)^{2} \theta^{2}(x) \, \dd \rhoSX(x).
        \end{aligned}
    \end{equation}
    For the first term in \eqref{eq: theta - hthetabuD}, note that both \( \phi \) and \( \hphimuD \) are bounded by \( D / (\alpha D + (1 - \alpha)) \) on \( \Omega_{D} \). By the Lipschitz property of the function \( t \mapsto (1 - \alpha) t / (1 - \alpha t) \) on \( [0, D / (\alpha D + (1 - \alpha))] \), we have
    \[
        \LE( \frac{(1 - \alpha) \phi(x)}{1 - \alpha \phi(x)} - \frac{(1 - \alpha) \hphimuD(x)}{1 - \alpha \hphimuD(x)} \RI)^{2} \le \frac{(\alpha D + (1 - \alpha))^{4}}{(1 - \alpha)^{2}} (\phi(x) - \hphimuD(x))^{2},
        \quad \forall x \in \Omega_{D}.
    \]
    Moreover, the truncation improves the estimate:
    \[
        |\phi(x) - \hphimuD(x)| \le |\phi(x) - \hphimu(x)|, \quad \forall x \in \Omega_{D}.
    \]
    Therefore,
    \begin{align*}
         & \eqspace \int_{\Omega_{D}} \LE( \frac{(1 - \alpha) \phi(x)}{1 - \alpha \phi(x)} - \frac{(1 - \alpha) \hphimuD(x)}{1 - \alpha \hphimuD(x)} \RI)^{2} \, \dd \rhoSX(x) \\
         & \le \frac{(\alpha D + (1 - \alpha))^{4}}{(1 - \alpha)^{2}} \int_{\Omega_{D}} (\phi(x) - \hphimu(x))^{2} \, \dd \rhoSX(x)                                            \\
         & \le \frac{(\alpha D + (1 - \alpha))^{4}}{(1 - \alpha)^{3}} \int_{\calX} (\phi(x) - \hphimu(x))^{2} \, ((1 - \alpha) \dd \rhoSX(x) + \alpha \dd \rhoTX(x))           \\
         & \le \frac{(\alpha D + (1 - \alpha))^{4}}{(1 - \alpha)^{3}} \norm{ \phi - \hphimu }_{\rhoRX}^{2}
        \le \frac{(\alpha D + (1 - \alpha))^{4}}{(1 - \alpha)^{3}} \LE( \varDelta_{\phi} \nU^{-\frac{\iota}{2 \iota + 1}} \log \frac{18}{\delta} \RI)^{2}.
    \end{align*}
    For the second term in \eqref{eq: theta - hthetabuD}, we use the moment condition from \aspref{asp: moment theta}. For any fixed \( m \in [2, M] \),
    \begin{align*}
         & \eqspace \int_{\calX \setminus \Omega_{D}} \LE( 1 - \frac{\hthetamuD(x)}{\theta(x)} \RI)^{2} \theta^{2}(x) \, \dd \rhoSX(x) \\
         & \le \int_{\calX \setminus \Omega_{D}} \theta^{2}(x) \, \dd \rhoSX(x)
        \le D^{-(m - 2)} \int_{\calX \setminus \Omega_{D}} \theta^{m}(x) \, \dd \rhoSX(x)
        \le D^{-(m - 2)} \int_{\calX} \theta^{m}(x) \, \dd \rhoSX(x)                                                                   \\
         & \le D^{-(m - 2)} \varXi_{m}.
    \end{align*}

    Together, these estimates imply
    \[
        \norm{ \theta - \hthetamuD }_{\rhoSX}^{2}
        \le \frac{(\alpha D + (1 - \alpha))^{4}}{(1 - \alpha)^{3}} \LE( \varDelta_{\phi} \nU^{-\frac{\iota}{2 \iota + 1}} \log \frac{18}{\delta} \RI)^{2} + D^{-(m - 2)} \varXi_{m}
    \]
    with probability at least \( 1 - \delta \). Setting the truncation threshold \( D \) as
    \[
        D = \nU^{\nu}, \quad \nu = \frac{1}{m} \cdot \frac{2 \iota}{2 \iota + 1},
    \]
    we balance the two error terms, obtaining
    \[
        \norm{ \theta - \hthetamuD }_{\rhoSX} \le \LE( (1 - \alpha)^{-3/2} \varDelta_{\phi} + \varXi_{m}^{1/2} \RI) \nU^{- \LE( \frac{\iota}{2 \iota + 1} - 2 \nu \RI)} \log \frac{18}{\delta}.
    \]
    Hence, the theorem holds.
\end{proofof}

\subsection{Proof of \thmref{thm: importance weighting}}
\label{sec: proof importance weighting}

In this section, we establish bounds on the convergence rate of the regression function estimator, as measured by \( \norm{ \frho - \hflam }_{\rhoTX} \) and \( \norm{ \frho - \hflam }_{\calH} \). To unify the analysis for both norms, we employ the concept of interpolation spaces \citep{Smale2003EstimatingAE, Steinwart2012MercerTG}.
\begin{definition}[Interpolation spaces]
    \label{def: interpolation}
    Let \( \LT \) be the integral operator associated with \( \rhoTX \). For \( \gamma \in [0, 1] \), the interpolation space \( \ranH{\gamma} \) is defined as
    \[
        \ranH{\gamma} = \operatorname{ran} \LT^{\gamma/2} = \family{ \LT^{\gamma/2} \, f: f \in \calL^{2}(\calX, \rhoTX) }.
    \]
    The inner product on \( \ranH{\gamma} \) is given by
    \[
        \inner{ \LT^{\gamma/2} \, f, \LT^{\gamma/2} \, g }_{\ranH{\gamma}}
        = \inner{ f, g }_{\rhoTX}.
    \]
\end{definition}
It follows that \( \ranH{1} = \calH \) and \( \ranH{0} \subseteq \calL^{2}(\calX, \rhoTX) \). Moreover, the following relation holds (analogous to \eqref{eq: rho-norm / H-norm}):
\[
    \norm{ \LT^{\frac{1 - \gamma}{2}} \, f }_{\calH} = \norm{ f }_{\ranH{\gamma}},
    \quad \forall f \in \ranH{\gamma}.
\]
Within this framework, it suffices to bound
\[
    \norm{ \frho - \hflam }_{\ranH{\gamma}},
\]
and then specialize to \( \gamma = 0 \) or \( 1 \) to complete the proof.

Our following analysis conditions on the high-probability event from \thmref{thm: density ratio}, which ensures that for \( \iota \in [1/2, \tau] \) and fixed \( m \ge 2 \), the density ratio estimator \( \htheta = \hthetamuD \) converges to \( \theta \) with probability at least \( 1 - \delta \) (denoted as \( \delta' \) below for clarity):
\begin{equation}
    \label{eq: theta bound}
    \norm{ \theta - \htheta }_{\rhoSX} \le \LE( (1 - \alpha)^{-3/2} \varDelta_{\phi} + \varXi_{m}^{1/2} \RI) \nU^{- \LE( \frac{\iota}{2 \iota + 1} - 2 \nu \RI)} \log \frac{18}{\delta'},
\end{equation}
where \( \alpha \in (0, 1) \) is the mixing coefficient, and \( \nu = 2 \iota / (m (2 \iota + 1)) \).

Defining the auxiliary function
\[
    \flam = \glam(\LT) \, \LT \, \frho,
\]
the error decomposition proceeds as follows:
\begin{equation}
    \label{eq: (frho) decompose 1}
    \norm{ \frho - \hflam }_{\ranH{\gamma}} \le \norm{ \frho - \flam }_{\ranH{\gamma}} + \norm{ \flam - \hflam }_{\ranH{\gamma}}.
\end{equation}
The second term is further decomposed as
\begin{equation}
    \label{eq: (frho) decompose 2}
    \begin{aligned}
        \norm{ \flam - \hflam }_{\ranH{\gamma}}
         & = \norm{ \LT^{\frac{1 - \gamma}{2}} \, (\flam - \hflam) }_{\calH}                                                                                                \\
         & \le \norm{ \LT^{\frac{1 - \gamma}{2}} \, \LWlam^{-1/2} } \cdot \norm{ \LWlam^{1/2} \, \hLWlam^{-1/2} } \cdot \norm{ \hLWlam^{1/2} \, (\flam - \hflam) }_{\calH},
    \end{aligned}
\end{equation}
where we define the population and empirical importance-weighted integral operators, \( \LW \) and \( \hLW \), respectively, as
\[
    \LW \, f = \int_{\calX} \htheta(x) f(x) \, K(\cdot, x) \, \dd \rhoSX(x),
    \quad \hLW \, f = \frac{1}{\nL} \sum_{i=1}^{\nL} \htheta(x_{i}) f(x_{i}) \, K(\cdot, x_{i}),
\]
and denote \( \LWlam = \LW + \lambda I \) and \( \hLWlam = \hLW + \lambda I \). Recalling that \( \hflam = \glam(\hLW) \, \hSWa \, \bsy \), we expand the last factor to
\begin{equation}
    \label{eq: (frho) decompose 3}
    \begin{aligned}
         & \eqspace \norm{ \hLWlam^{1/2} \, (\flam - \hflam) }_{\calH}                                                                                                  \\
         & \le \norm{ \hLWlam^{1/2} \, \glam(\hLW) \, (\hLW \, \flam - \hSWa \, \bsy) }_{\calH} + \norm{ \hLWlam^{1/2} \, (I - \glam(\hLW) \, \hLW) \, \flam }_{\calH}.
    \end{aligned}
\end{equation}
Combining \eqref{eq: (frho) decompose 1}, \eqref{eq: (frho) decompose 2}, and \eqref{eq: (frho) decompose 3}, we obtain
\[
    \norm{ \frho - \hflam }_{\ranH{\gamma}} \le Q_{1} + Q_{2} Q_{3} (Q_{4} + Q_{5}),
\]
where the terms are defined as
\begin{equation}
    \label{eq: decompose Q}
    \begin{alignedat}{2}
        Q_{1} & = \norm{ \frho - \flam }_{\ranH{\gamma}},
              & \quad Q_{2}                                                             & = \norm{ \LT^{\frac{1 - \gamma}{2}} \, \LWlam^{-1/2} },                             \\
        Q_{3} & = \norm{ \LWlam^{1/2} \, \hLWlam^{-1/2} },
              & Q_{4}                                                                   & = \norm{ \hLWlam^{1/2} \, \glam(\hLW) \, (\hLW \, \flam - \hSWa \, \bsy) }_{\calH}, \\
        Q_{5} & = \norm{ \hLWlam^{1/2} \, (I - \glam(\hLW) \, \hLW) \, \flam }_{\calH}.
    \end{alignedat}
\end{equation}
In the following, we bound \( Q_{1} \), \( Q_{2} \), \( Q_{3} \), \( Q_{4} \), and \( Q_{5} \) via separate propositions. To begin, \proref{pro: Q1} establishes an upper bound for \( Q_{1} \) in \eqref{eq: decompose Q}.
\begin{proposition}
    \label{pro: Q1}
    Suppose that \aspref{asp: source frho} holds with \( r \in [1/2, \tau] \). Then, for any \( \gamma \in [0, 1] \), the following bound holds:
    \[
        Q_{1} = \norm{ \frho - \flam }_{\ranH{\gamma}} \le F \norm{ \urho }_{\rhoTX} \lambda^{r - \frac{\gamma}{2}}.
    \]
\end{proposition}
\begin{proof}
    Following reasoning similar to that in the proof of \proref{pro: P1}, we derive
    \begin{align*}
        Q_{1} & = \norm{ \frho - \flam }_{\ranH{\gamma}}
        = \norm{ \LT^{\frac{1 - \gamma}{2}} \, (\frho - \flam) }_{\calH}
        = \norm{ \LT^{\frac{1 - \gamma}{2}} \, (I - \glam(\LT) \, \LT) \, \LT^{r} \, \urho }_{\calH}                         \\
              & \le \norm{ \LT^{r - \frac{\gamma}{2}} \, (I - \glam(\LT) \, \LT) } \cdot \norm{ \LT^{1/2} \, \urho }_{\calH}
        \le \sup_{t \in [0, \kappa^{2}]} t^{r - \frac{\gamma}{2}} |1 - t \glam(t)| \cdot \norm{ \urho }_{\rhoTX}             \\
              & \le F \norm{ \urho }_{\rhoTX} \lambda^{r - \frac{\gamma}{2}}.
    \end{align*}
    This completes the proof.
\end{proof}

We bound the term \( Q_{2} \) in \eqref{eq: decompose Q} using the following proposition.
\begin{proposition}
    \label{pro: Q2}
    Assume that the error bound \eqref{eq: theta bound} holds for the density ratio estimator \( \htheta \). Let the regularization parameter be chosen as \( \lambda = \nL^{-s} \) for some \( s > 0 \). If the sample sizes \( \nU \) and \( \nL \) satisfy
    \[
        \nU^{\frac{\iota}{2 \iota + 1} - 2 \nu} \ge 2 \kappa^{2} \LE( (1 - \alpha)^{-3/2} \varDelta_{\phi} + \varXi_{m}^{1/2} \RI) \log \frac{18}{\delta'} \cdot \nL^{s},
    \]
    then for any \( \gamma \in [0, 1] \), we have
    \[
        Q_{2} = \norm{ \LT^{\frac{1 - \gamma}{2}} \, \LWlam^{-1/2} } \le \sqrt{2} \lambda^{-\gamma/2}.
    \]
\end{proposition}
\begin{proof}
    We decompose \( Q_{2} \) by splitting the norm into
    \[
        Q_{2} = \norm{ \LT^{\frac{1 - \gamma}{2}} \, \LWlam^{-1/2} } \le \norm{ \LT^{\frac{1 - \gamma}{2}} \, \LTlam^{-1/2} } \cdot \norm{ \LTlam^{1/2} \, \LWlam^{-1/2} },
    \]
    where \( \LTlam = \LT + \lambda I \). The first factor is bounded using \lemref{lem: sup fraction}:
    \begin{equation}
        \label{eq: (Q2) component 1}
        \norm{ \LT^{\frac{1 - \gamma}{2}} \, \LTlam^{-1/2} } \le \sup_{t \ge 0} \LE( \frac{t^{1 - \gamma}}{t + \lambda} \RI)^{1/2} \le \lambda^{-\gamma/2}.
    \end{equation}
    For the second factor, we observe that \( \LWlam^{-1} \) admits the representation
    \begin{align*}
        \LWlam^{-1}
         & = (\LW + \lambda I)^{-1}
        = (\LW - \LT + \LT + \lambda I)^{-1}                             \\
         & = \LE( \LTlam - (\LT - \LW) \RI)^{-1}
        = \LE( \LE( I - (\LT - \LW) \, \LTlam^{-1} \RI) \LTlam \RI)^{-1} \\
         & = \LTlam^{-1} \LE( I - (\LT - \LW) \, \LTlam^{-1} \RI)^{-1}.
    \end{align*}
    Applying \lemref{lem: cordes} yields
    \begin{equation}
        \label{eq: (Q2) component 2}
        \norm{ \LTlam^{1/2} \, \LWlam^{-1/2} } \le \norm{ \LTlam \, \LWlam^{-1} }^{1/2} = \norm{ \LE( I - (\LT - \LW) \, \LTlam^{-1} \RI)^{-1} }^{1/2}.
    \end{equation}
    Now, we bound the operator norm \( \norm{ (\LT - \LW) \, \LTlam^{-1} } \):
    \[
        \norm{ (\LT - \LW) \, \LTlam^{-1} } \le \int_{\calX} |\theta(x) - \htheta(x)| \cdot \norm{ \Kx \otimes \Kx \, \LTlam^{-1} } \, \dd \rhoSX(x),
    \]
    where the operator \( \Kx \otimes \Kx \) is defined in \eqref{eq: Kx otimes Kx}. For any \( x \in \calX \), the integrand satisfies
    \[
        \norm{ \Kx \otimes \Kx \, \LTlam^{-1} }
        \le \norm{ \Kx \otimes \Kx } \cdot \norm{ \LTlam^{-1} }
        \le \kappa^{2} \lambda^{-1}.
    \]
    Applying the H\"{o}lder's inequality and the assumed bound \eqref{eq: theta bound}, we obtain
    \begin{align*}
        \norm{ (\LT - \LW) \, \LTlam^{-1} }
         & \le \kappa^{2} \lambda^{-1} \cdot \norm{ \theta - \htheta }_{\rhoSX}                                                                                                 \\
         & \le \kappa^{2} \LE( (1 - \alpha)^{-3/2} \varDelta_{\phi} + \varXi_{m}^{1/2} \RI) \nU^{-\LE( \frac{\iota}{2 \iota + 1} - 2 \nu \RI)} \nL^{s} \log \frac{18}{\delta'},
    \end{align*}
    where \( \nu = 2 \iota / (m (2 \iota + 1)) \). Provided the sample sizes \( \nU \) and \( \nL \) satisfy the condition
    \[
        \nU^{\frac{\iota}{2 \iota + 1} - 2 \nu} \ge 2 \kappa^{2} \LE( (1 - \alpha)^{-3/2} \varDelta_{\phi} + \varXi_{m}^{1/2} \RI) \log \frac{18}{\delta'} \cdot \nL^{s},
    \]
    it follows that
    \begin{equation}
        \label{eq: sample bound 1}
        \norm{ (\LT - \LW) \, \LTlam^{-1} } \le \frac{1}{2}.
    \end{equation}
    Consequently, the Neumann series gives
    \begin{equation}
        \label{eq: (Q2) component 2'}
        \norm{ \LE( I - (\LT - \LW) \, \LTlam^{-1} \RI)^{-1} }
        \le \sum_{\ell=0}^{\infty} \norm{ (\LT - \LW) \, \LTlam^{-1} }^{\ell} \le 2.
    \end{equation}

    Combining the bounds from \eqref{eq: (Q2) component 1}, \eqref{eq: (Q2) component 2}, and \eqref{eq: (Q2) component 2'}, we conclude
    \[
        Q_{2} = \norm{ \LT^{\frac{1 - \gamma}{2}} \, \LWlam^{-1/2} } \le \norm{ \LT^{\frac{1 - \gamma}{2}} \, \LTlam^{-1/2} } \cdot \norm{ \LE( I - (\LT - \LW) \, \LTlam^{-1} \RI)^{-1} }^{1/2} \le \sqrt{2} \lambda^{-\gamma/2}.
    \]
    This establishes the proposition.
\end{proof}

The following proposition serves to bound \( Q_{3} \) in \eqref{eq: decompose Q}.
\begin{proposition}
    \label{pro: Q3}
    Assume that the error bound \eqref{eq: theta bound} holds for the density ratio estimator \( \htheta \). For any \( \delta \in (0, 1) \), if the regularization parameter \( \lambda = \nL^{-s} \) satisfies \( 0 < s < 1 / (2 r + 1) \), and the sample sizes \( \nU \) and \( \nL \) satisfy
    \[
        \LE\{
        \begin{aligned}
             & \nU^{\frac{\iota}{2 \iota + 1} - 2 \nu} \ge \frac{2 \kappa^{2}}{\norm{ \LT }} \LE( (1 - \alpha)^{-3/2} \varDelta_{\phi} + \varXi_{m}^{1/2} \RI) \log \frac{18}{\delta'}, \\
             & \nL^{1 - s (2 r + 1)} \ge 14 \kappa^{2} \LE( \log \frac{4 \kappa^{2} (\norm{ \LT } + 2)}{\norm{ \LT } \delta} + 1 + \frac{1}{2 r} \RI) \nU^{2 \nu},
        \end{aligned}
        \RI.
    \]
    where \( \nu = 2 \iota / (m (2 \iota + 1)) \), then with probability at least \( 1 - \delta \), we have
    \[
        Q_{3} = \norm{ \LWlam^{1/2} \, \hLWlam^{-1/2} } \le \sqrt{2}.
    \]
\end{proposition}
\begin{proof}
    We begin by bounding the operator norm \( \norm{ \LWlam^{-1/2} \, (\LW - \hLW) \, \LWlam^{-1/2} } \). Define the random variable \( A(x) = \LWlam^{-1/2} \, (\LW - \htheta(x) \, \Kx \otimes \Kx) \, \LWlam^{-1/2} \) for \( x \sim \rhoSX \), with \( A_{i} = A(x_{i}) \). Then,
    \[
        \norm{ \LWlam^{-1/2} \, (\LW - \hLW) \, \LWlam^{-1/2} } = \norm{ \frac{1}{\nL} \sum_{i=1}^{\nL} A_{i} },
    \]
    and we apply \lemref{lem: concentration operator} to control this empirical average. The analysis proceeds in three steps:
    \begin{enumerate}
        \item The truncation in the density ratio estimation ensures \( \sup_{x \in \calX} |\htheta(x)| \le D = \nU^{\nu} \), where \( \nu = 2 \iota / (m (2 \iota + 1)) \). For any \( x \in \calX \), we have
              \begin{align*}
                   & \eqspace \norm{ \LWlam^{-1/2} \, (\htheta(x) \, \Kx \otimes \Kx) \, \LWlam^{-1/2} }                    \\
                   & \le |\htheta(x)| \cdot \norm{ \LWlam^{-1/2} \, \Kx \otimes \Kx \, \LWlam^{-1/2} }
                  \le |\htheta(x)| \cdot \norm{ \LWlam^{-1/2} } \cdot \norm{ \Kx \otimes \Kx } \cdot \norm{ \LWlam^{-1/2} } \\
                   & \le \kappa^{2} \lambda^{-1} \nU^{\nu}.
              \end{align*}
              Consequently,
              \begin{equation}
                  \label{eq: (Q3) A}
                  \norm{ A(x) } \le 2 \kappa^{2} \lambda^{-1} \nU^{\nu},
                  \quad \forall x \in \calX.
              \end{equation}

        \item We establish the bound
              \begin{equation}
                  \label{eq: (Q3) E A^2}
                  \begin{aligned}
                      \EE[x \sim \rhoSX]{ A^{2}(x) }
                       & \preceq \EE[x \sim \rhoSX]{ \LE( \LWlam^{-1/2} \, (\htheta(x) \, \Kx \otimes \Kx) \, \LWlam^{-1/2} \RI)^{2} }                           \\
                       & \preceq \kappa^{2} \lambda^{-1} \nU^{\nu} \cdot \EE[x \sim \rhoSX]{ \LWlam^{-1/2} \, (\htheta(x) \, \Kx \otimes \Kx) \, \LWlam^{-1/2} } \\
                       & = \kappa^{2} \lambda^{-1} \nU^{\nu} \cdot \LWlam^{-1/2} \, \LW \, \LWlam^{-1/2}.
                  \end{aligned}
              \end{equation}
              The following upper bound is immediate:
              \begin{equation}
                  \label{eq: (Q3) LWlam^-1/2 LW LWlam^-1/2 upper}
                  \norm{ \LWlam^{-1/2} \, \LW \, \LWlam^{-1/2} } = \frac{\norm{ \LW }}{\norm{ \LW } + \lambda} \le 1.
              \end{equation}
              For the lower bound, note that
              \begin{equation}
                  \label{eq: LW - LT}
                  \begin{aligned}
                      \norm{ \LW - \LT }
                       & \le \int_{\calX} |\htheta(x) - \theta(x)| \norm{ \Kx \otimes \Kx } \, \dd \rhoSX
                      \le \kappa^{2} \norm{ \htheta - \theta }_{\rhoSX}                                                                                                               \\
                       & \le \kappa^{2} \LE( (1 - \alpha)^{-3/2} \varDelta_{\phi} + \varXi_{m}^{1/2} \RI) \nU^{-\LE( \frac{\iota}{2 \iota + 1} - 2 \nu \RI)} \log \frac{18}{\delta'}.
                  \end{aligned}
              \end{equation}
              Under the sample size condition
              \[
                  \nU^{\frac{\iota}{2 \iota + 1} - 2 \nu} \ge \frac{2 \kappa^{2}}{\norm{ \LT }} \LE( (1 - \alpha)^{-3/2} \varDelta_{\phi} + \varXi_{m}^{1/2} \RI) \log \frac{18}{\delta'},
              \]
              we obtain the bound
              \begin{equation}
                  \label{eq: sample bound 2}
                  \norm{ \LW - \LT } \le \frac{\norm{ \LT }}{2},
              \end{equation}
              which implies \( \norm{ \LW } \ge \norm{ \LT } / 2 \) and consequently yields the lower estimate
              \begin{equation}
                  \label{eq: (Q3) LWlam^-1/2 LW LWlam^-1/2 lower}
                  \norm{ \LWlam^{-1/2} \, \LW \, \LWlam^{-1/2} } = \frac{\norm{ \LW }}{\norm{ \LW } + \lambda} \ge \frac{\norm{ \LT }}{\norm{ \LT } + 2 \lambda}.
              \end{equation}

        \item By the uniform bound for \( \htheta(x) \), we obtain
              \[
                  \LW \preceq \sup_{x \in \calX} |\htheta(x)| \cdot \LS \preceq \nU^{\nu} \cdot \LS,
              \]
              which leads to
              \begin{equation}
                  \label{eq: (Q3) LWlam^-1/2 LW LWlam^-1/2 trace}
                  \begin{aligned}
                      \Tr{ \LWlam^{-1/2} \, \LW \, \LWlam^{-1/2} }
                       & = \Tr{ \LWlam^{-1} \, \LW }
                      \le \lambda^{-1} \Tr{ \LW }
                      \le \lambda^{-1} \nU^{\nu} \Tr{ \LS }     \\
                       & \le \kappa^{2} \lambda^{-1} \nU^{\nu},
                  \end{aligned}
              \end{equation}
              where the last inequality follows from \( \Tr{ \LS } \le \kappa^{2} \).
    \end{enumerate}
    Applying \lemref{lem: concentration operator} with the bounds \eqref{eq: (Q3) A}, \eqref{eq: (Q3) E A^2}, \eqref{eq: (Q3) LWlam^-1/2 LW LWlam^-1/2 upper}, \eqref{eq: (Q3) LWlam^-1/2 LW LWlam^-1/2 lower}, and \eqref{eq: (Q3) LWlam^-1/2 LW LWlam^-1/2 trace}, we obtain with probability at least \( 1 - \delta \) that
    \begin{align*}
        \norm{ \LWlam^{-1/2} \, (\LW - \hLW) \, \LWlam^{-1/2} }
         & \le \frac{4}{3} \kappa^{2} \lambda^{-1} \nU^{\nu} \nL^{-1} h + \sqrt{2 \kappa^{2} \lambda^{-1} \nU^{\nu} \nL^{-1} h} \\
         & = \frac{4}{3} \kappa^{2} \nU^{\nu} \nL^{s - 1} h + \sqrt{2 \kappa^{2} \nU^{\nu} \nL^{s - 1} h},
    \end{align*}
    where \( h \) is given by (note that \( \lambda = \nL^{-s} \))
    \begin{align*}
        h & = \log \LE( 4 \kappa^{2} \lambda^{-1} \nU^{\nu} \MID/ \LE( \frac{\norm{ \LT }}{\norm{ \LT } + 2 \lambda} \delta \RI) \RI)
        = \log \LE( 4 \kappa^{2} \nU^{\nu} \nL^{s} \MID/ \LE( \frac{\norm{ \LT }}{\norm{ \LT } + 2 \lambda} \delta \RI) \RI)          \\
          & \le \log \frac{4 \kappa^{2} (\norm{ \LT } + 2)}{\norm{ \LT } \delta} + \nU^{\nu} + \frac{1}{2 r} \nL^{2 r s}
        \le \LE( \log \frac{4 \kappa^{2} (\norm{ \LT } + 2)}{\norm{ \LT } \delta} + 1 + \frac{1}{2 r} \RI) \nU^{\nu} \nL^{2 r s}.
    \end{align*}
    If \( 0 < s < 1 / (2 r + 1) \), and the sample sizes \( \nU \) and \( \nL \) satisfy
    \[
        \nL^{1 - s (2 r + 1)} \ge 14 \kappa^{2} \LE( \log \frac{4 \kappa^{2} (\norm{ \LT } + 2)}{\norm{ \LT } \delta} + 1 + \frac{1}{2 r} \RI) \nU^{2 \nu},
    \]
    we ensure
    \begin{equation}
        \label{eq: sample bound 3}
        \norm{ \LWlam^{-1/2} \, (\LW - \hLW) \, \LWlam^{-1/2} } \le \frac{1}{10} + \frac{2}{5} = \frac{1}{2}.
    \end{equation}
    Finally, we bound \( Q_{3} \) as follows:
    \begin{align*}
        Q_{3}^{2} & = \norm{ \LWlam^{1/2} \, \hLWlam^{-1/2} }^{2}
        = \norm{ \LWlam^{1/2} \, \hLWlam^{-1} \, \LWlam^{1/2} }
        = \norm{ \LE( \LWlam^{-1/2} \, \hLWlam \, \LWlam^{-1/2} \RI)^{-1} }                                   \\
                  & \le \sum_{\ell=0}^{\infty} \norm{ \LWlam^{-1/2} \, (\LW - \hLW) \, \LWlam^{-1/2} }^{\ell}
        \le 2,
    \end{align*}
    which implies \( Q_{3} \le \sqrt{2} \) with probability at least \( 1 - \delta \).
\end{proof}

We now establish a bound for \( Q_{4} \) in \eqref{eq: decompose Q} via the following proposition.
\begin{proposition}
    \label{pro: Q4}
    Assume that the error bound \eqref{eq: theta bound} holds for the density ratio estimator \( \htheta \). Suppose that \aspref{asp: source frho} holds with \( r \in [1/2, \tau] \), \aspref{asp: moment noise} holds with \( G_{0}, \sigma_{0} > 0 \), and condition on the event that the bound for \( Q_{3} \) given in \proref{pro: Q3} holds. For any \( \delta \in (0, 1) \), if, in addition to the conditions in \proref{pro: Q3}, the regularization parameter \( \lambda = \nL^{-s} \) satisfies \( 0 < s < 1 / (2 r + 1) \), and the sample sizes \( \nU \) and \( \nL \) satisfy
    \[
        \LE\{
        \begin{aligned}
             & \nU^{\nu} \le \nL^{\frac{1}{2} - s \LE( r + \frac{1}{2} \RI)}, \\
             & \nL^{s} \le \nU^{\frac{\iota}{2 \iota + 1} - 2 \nu},
        \end{aligned}
        \RI.
    \]
    where \( \nu = 2 \iota / (m (2 \iota + 1)) \), then
    \begin{align*}
        Q_{4} & = \norm{ \hLWlam^{1/2} \, \glam(\hLW) \, (\hLW \, \flam - \hSWa \, \bsy) }_{\calH}                                                                                                                                           \\
              & \le 2 \sqrt{2} E \Bigg( 4 \kappa \LE( \kappa F \norm{ \urho }_{\rhoTX} + 2 G_{0} + \sqrt{2} \sigma_{0} \RI) \log \frac{2}{\delta}                                                                                            \\
              & \hphantom{\eqspace 2 \sqrt{2} E \Bigg( \,} + F \norm{ \urho }_{\rhoTX} \LE( \LE( \kappa^{2} \LE( (1 - \alpha)^{-3/2} \varDelta_{\phi} + \varXi_{m}^{1/2} \RI) \log \frac{18}{\delta'} \RI)^{1/2} + 1 \RI) \Bigg) \lambda^{r}
    \end{align*}
    with probability at least \( 1 - \delta \).
\end{proposition}
\begin{proof}
    Using the filter function property \eqref{eq: glam} and the operator norm bounds from \proref{pro: Q3}, we decompose \( Q_{4} \) as
    \begin{align*}
        Q_{4} & = \norm{ \hLWlam^{1/2} \, \glam(\hLW) \, (\hLW \, \flam - \hSWa \, \bsy) }_{\calH}                                                                                               \\
              & \le \norm{ \hLWlam^{1/2} \, \glam(\hLW) \, \hLWlam^{1/2} } \cdot \norm{ \hLWlam^{-1/2} \, \LWlam^{1/2} } \cdot \norm{ \LWlam^{-1/2} \, (\hLW \, \flam - \hSWa \, \bsy) }_{\calH} \\
              & \le 2 \sqrt{2} E \norm{ \LWlam^{-1/2} \, (\hLW \, \flam - \hSWa \, \bsy) }_{\calH}.
    \end{align*}
    Moreover, we have the decomposition
    \begin{equation}
        \label{eq: (Q4) decompose}
        \begin{aligned}
             & \eqspace \norm{ \LWlam^{-1/2} \, (\hLW \, \flam - \hSWa \, \bsy) }_{\calH}                                                                                                      \\
             & \le \norm{ \LWlam^{-1/2} \LE( (\hLW \, \flam - \hSWa \, \bsy) - (\LW \, \flam - \LW \, \frho) \RI) }_{\calH} + \norm{ \LWlam^{-1/2} \, (\LW \, \flam - \LW \, \frho) }_{\calH}.
        \end{aligned}
    \end{equation}

    For the first term in \eqref{eq: (Q4) decompose}, define the random variable
    \[
        \xi(x, y) = \LWlam^{-1/2} \LE( \htheta(x) (\Kx \, y - \Kx \otimes \Kx \, \flam) \RI) = \htheta(x) (y - \flam(x)) \cdot \LWlam^{-1/2} \, K(\cdot, x)
    \]
    for \( (x, y) \sim \rhoS_{\calX \times \calY} \), with \( \xi_{i} = \xi(x_{i}, y_{i}) \). Then,
    \[
        \norm{ \LWlam^{-1/2} \LE( (\hLW \, \flam - \hSWa \, \bsy) - (\LW \, \flam - \LW \, \frho) \RI) }_{\calH} = \norm{ \frac{1}{\nL} \sum_{i=1}^{\nL} \xi_{i} - \EE[(x, y) \sim \rhoS_{\calX \times \calY}]{ \xi(x, y) } }_{\calH}.
    \]
    We apply \lemref{lem: concentration vector} to control this empirical average. To satisfy the lemma's conditions, we bound the moments \( \EE{ \norm{ \xi }_{\calH}^{\ell} } \) for \( \ell = 2, 3, \dots \) as follows:
    \begin{align*}
        \EE[(x, y) \sim \rhoS_{\calX \times \calY}]{ \norm{ \xi(x, y) }_{\calH}^{\ell} }
         & \le \EE[(x, y) \sim \rhoS_{\calX \times \calY}]{ \LE( |y - \flam(x)| \cdot \sup_{x \in \calX} \LE( |\htheta(x)| \cdot \norm{ \LWlam^{-1/2} \, K(\cdot, x) }_{\calH} \RI) \RI)^{\ell} } \\
         & \le \LE( \kappa \lambda^{-1/2} \nU^{\nu} \RI)^{\ell} \cdot \EE[(x, y) \sim \rhoS_{\calX \times \calY}]{ |y - \flam(x)|^{\ell} }.
    \end{align*}
    In the last inequality, we use the bounds
    \[
        \sup_{x \in \calX} \norm{ \LWlam^{-1/2} \, K(\cdot, x) }_{\calH} \le \kappa \lambda^{-1/2},
        \quad \sup_{x \in \calX} |\htheta(x)| \le \nU^{\nu},
    \]
    where \( \nu = 2 \iota / (m (2 \iota + 1)) \). We further derive
    \begin{align*}
        \EE[(x, y) \sim \rhoS_{\calX \times \calY}]{ |y - \flam(x)|^{\ell} }
         & \overset{\text{(a)}}{\le} 2^{\ell - 1} \EE[x \sim \rhoSX]{ |\frho(x) - \flam(x)|^{\ell} + \EE[y \sim \rho_{\calY \mid \calX}]{ |y - \frho(x)|^{\ell} \MID| x } }                \\
         & \overset{\text{(b)}}{\le} 2^{\ell - 1} \LE( \LE( \kappa F \norm{ \urho }_{\rhoTX} \lambda^{r - \frac{1}{2}} \RI)^{\ell} + \frac{1}{2} \ell! G_{0}^{\ell-2} \sigma_{0}^{2} \RI).
    \end{align*}
    Here, step (a) applies the convexity inequality \( |a + b|^{\ell} \le 2^{\ell - 1} (|a|^{\ell} + |b|^{\ell}) \); in step (b), we invoke \aspref{asp: moment noise} and the error bound from \proref{pro: Q1} (with \( \gamma = 1 \)), which gives
    \[
        |\flam(x) - \frho(x)| = \LE| \inner{ \flam - \frho, K(\cdot, x) }_{\calH} \RI| \le \kappa \norm{ \flam - \frho }_{\calH} \le \kappa F \norm{ \urho }_{\rhoTX} \lambda^{r - \frac{1}{2}}.
    \]
    Together, we obtain
    \[
        \EE[(x, y) \sim \rhoS_{\calX \times \calY}]{ \norm{ \xi(x, y) }_{\calH}^{\ell} } \le \frac{1}{2} \ell! G^{\ell - 2} \sigma^{2}
    \]
    with
    \[
        G = \kappa \LE( \kappa F \norm{ \urho }_{\rhoTX} \lambda^{r - \frac{1}{2}} + 2 G_{0} \RI) \lambda^{-1/2} \nU^{\nu},
        \quad \sigma = \kappa \LE( \kappa F \norm{ \urho }_{\rhoTX} \lambda^{r - \frac{1}{2}} + \sqrt{2} \sigma_{0} \RI) \lambda^{-1/2} \nU^{\nu}.
    \]
    Thus, by \lemref{lem: concentration vector}, with probability at least \( 1 - \delta \), we obtain
    \begin{equation}
        \label{eq: (Q4) component 1}
        \begin{aligned}
             & \eqspace \norm{ \LWlam^{-1/2} \LE( (\hLW \, \flam - \hSWa \, \bsy) - (\LW \, \flam - \LW \, \frho) \RI) }_{\calH}                                                                                                                           \\
             & \le 4 \LE( \frac{G}{\nL} + \frac{\sigma}{\sqrt{\nL}} \RI) \log \frac{2}{\delta}           \le 4 \kappa \LE( \kappa F \norm{ \urho }_{\rhoTX} + 2 G_{0} + \sqrt{2} \sigma_{0} \RI) \lambda^{-1/2} \nU^{\nu} \nL^{-1/2} \log \frac{2}{\delta} \\
             & = 4 \kappa \LE( \kappa F \norm{ \urho }_{\rhoTX} + 2 G_{0} + \sqrt{2} \sigma_{0} \RI) \lambda^{r} \nU^{\nu} \nL^{- \frac{1}{2} + s \LE( r + \frac{1}{2} \RI)} \log \frac{2}{\delta}.
        \end{aligned}
    \end{equation}

    For the second term in \eqref{eq: (Q4) decompose}, we decompose it as
    \begin{align*}
         & \eqspace \norm{ \LWlam^{-1/2} \, (\LW \, \flam - \LW \, \frho) }_{\calH}
        \le \norm{ \LWlam^{-1/2} \, \LW^{1/2} } \cdot \norm{ \LW^{1/2} \, (\flam - \frho) }_{\calH}                              \\
         & \le \norm{ \LW^{1/2} \, (\flam - \frho) }_{\calH}
        \le \norm{ \LW^{1/2} - \LT^{1/2} } \cdot \norm{ \flam - \frho }_{\calH} + \norm{ \LT^{1/2} \, (\flam - \frho) }_{\calH}  \\
         & \le F \norm{ \urho }_{\rhoTX} \LE( \norm{ \LW^{1/2} - \LT^{1/2} } \cdot \lambda^{r - \frac{1}{2}} + \lambda^{r} \RI),
    \end{align*}
    where the last inequality uses \proref{pro: Q1} with \( \gamma = 1 \) and 0. By \lemref{lem: A^c - B^c} and the bound on \( \norm{ \LW - \LT } \) from \eqref{eq: LW - LT}, we have
    \begin{align*}
         & \eqspace \norm{ \LW^{1/2} - \LT^{1/2} } \le \norm{ \LW - \LT }^{1/2}                                                                                                         \\
         & \le \LE( \kappa^{2} \LE( (1 - \alpha)^{-3/2} \varDelta_{\phi} + \varXi_{m}^{1/2} \RI) \nU^{-\LE( \frac{\iota}{2 \iota + 1} - 2 \nu \RI)} \log \frac{18}{\delta'} \RI)^{1/2}.
    \end{align*}
    Substituting this yields
    \begin{equation}
        \label{eq: (Q4) component 2}
        \begin{aligned}
             & \eqspace \norm{ \LWlam^{-1/2} \, (\LW \, \flam - \LW \, \frho) }_{\calH}                                                                                                                                                                                                  \\
             & \le F \norm{ \urho }_{\rhoTX} \cdot \LE( \LE( \kappa^{2} \LE( (1 - \alpha)^{-3/2} \varDelta_{\phi} + \varXi_{m}^{1/2} \RI) \log \frac{18}{\delta'} \RI)^{1/2} \lambda^{r} \nU^{- \frac{1}{2} \LE( \frac{\iota}{2 \iota + 1} - 2 \nu \RI)} \nL^{s / 2} + \lambda^{r} \RI).
        \end{aligned}
    \end{equation}

    Under the condition \( \lambda = \nL^{-s} \) with \( 0 < s < 1 / (2 r + 1) \), if the sample sizes \( \nU \) and \( \nL \) satisfy
    \[
        \LE\{
        \begin{aligned}
             & \nU^{\nu} \le \nL^{\frac{1}{2} - s \LE( r + \frac{1}{2} \RI)}, \\
             & \nL^{s} \le \nU^{\frac{\iota}{2 \iota + 1} - 2 \nu},
        \end{aligned}
        \RI.
    \]
    we ensure the relations
    \begin{equation}
        \label{eq: sample bound 4}
        \lambda^{r} \nU^{\nu} \nL^{- \frac{1}{2} + s \LE( r + \frac{1}{2} \RI)} \le \lambda^{r},
        \quad \lambda^{r} \nU^{- \frac{1}{2} \LE( \frac{\iota}{2 \iota + 1} - 2 \nu \RI)} \nL^{s / 2} \le \lambda^{r}.
    \end{equation}
    Combining \eqref{eq: (Q4) component 1} and \eqref{eq: (Q4) component 2} via \eqref{eq: (Q4) decompose}, we conclude
    \begin{align*}
        Q_{4} & \le 2 \sqrt{2} E \Bigg( 4 \kappa \LE( \kappa F \norm{ \urho }_{\rhoTX} + 2 G_{0} + \sqrt{2} \sigma_{0} \RI) \log \frac{2}{\delta}                                                                                             \\
              & \hphantom{\eqspace 2 \sqrt{2} E \Bigg( \,} + F \norm{ \urho }_{\rhoTX} \LE( \LE( \kappa^{2} \LE( (1 - \alpha)^{-3/2} \varDelta_{\phi} + \varXi_{m}^{1/2} \RI) \log \frac{18}{\delta'} \RI)^{1/2} + 1 \RI) \Bigg) \lambda^{r}.
    \end{align*}
    with probability at least \( 1 - \delta \).
\end{proof}

Finally, we bound \( Q_{5} \) in \eqref{eq: decompose Q} using \proref{pro: Q5}.
\begin{proposition}
    \label{pro: Q5}
    Assume that the error bound \eqref{eq: theta bound} holds for the density ratio estimator \( \htheta \). Suppose that \aspref{asp: source frho} holds with \( r \in [1/2, \tau] \), and condition on the event that the bound for \( Q_{3} \) given in \proref{pro: Q3} holds. For any \( \delta \in (0, 1) \), if, in addition to the conditions in \proref{pro: Q3}, the regularization parameter \( \lambda = \nL^{-s} \) satisfy \( 0 < s < \min \family{ 1 / 2, 1 / (2 r - 1) } \), and the sample sizes \( \nU \) and \( \nL \) satisfy
    \[
        \LE\{
        \begin{aligned}
             & \nL^{s} \le \nU^{\frac{\iota}{2 \iota + 1} \cdot \min \family{ 1, \frac{2}{2 r - 1} } - \nu \LE( \min \family{ 1, \frac{2}{2 r - 1} } + 1 \RI)}, \\
             & \nU^{2 \nu} \le \nL^{\frac{1}{2} \min \family{ 1, \frac{2}{2 r - 1} } - s},
        \end{aligned}
        \RI.
    \]
    where \( \nu = 2 \iota / (m (2 \iota + 1)) \), then
    \begin{align*}
        Q_{5} & = \norm{ \hLWlam^{1/2} \, (I - \glam(\hLW) \, \hLW) \, \flam }_{\calH}                                                                                                                                            \\
              & \le 2 E F \norm{ \urho }_{\rhoTX} \LE( 2 \kappa^{2 r - 1} \LE( (1 - \alpha)^{-3/2} \varDelta_{\phi} + \varXi_{m}^{1/2} + 10 \RI) r \LE( \log \frac{18}{\delta'} + \log \frac{2}{\delta} \RI) + 1 \RI) \lambda^{r}
    \end{align*}
    with probability at least \( 1 - \delta \).
\end{proposition}
\begin{proof}
    We begin with the expansion
    \begin{align*}
        Q_{5}
         & = \norm{ \hLWlam^{1/2} \, (I - \glam(\hLW) \, \hLW) \, \flam }_{\calH}
        = \norm{ \hLWlam^{1/2} \, (I - \glam(\hLW) \, \hLW) \, \glam(\LT) \, \LT \, \frho }_{\calH}                                                                     \\
         & = \norm{ \hLWlam^{1/2} \, (I - \glam(\hLW) \, \hLW) \, \glam(\LT) \, \LT^{r+1} \, \urho }_{\calH}                                                            \\
         & \le \norm{ \hLWlam^{1/2} \, (I - \glam(\hLW) \, \hLW) \, \LT^{r - \frac{1}{2}} } \cdot \norm{ \glam(\LT) \, \LT } \cdot \norm{ \LT^{1/2} \, \urho }_{\calH}.
    \end{align*}
    We bound the second and third factors as follows:
    \[
        \norm{ \glam(\LT) \, \LT } \le \sup_{t \in [0, \kappa^{2}]} t \gmu(t) \le E,
        \quad \norm{ \LT^{1/2} \, \urho }_{\calH} = \norm{ \urho }_{\rhoTX}.
    \]
    We further decompose the first term:
    \begin{align*}
         & \eqspace \norm{ \hLWlam^{1/2} \, (I - \glam(\hLW) \, \hLW) \, \LT^{r - \frac{1}{2}} }                                                                                                                 \\
         & \le \norm{ \hLWlam^{1/2} \, (I - \glam(\hLW) \, \hLW) } \cdot \norm{ \LT^{r - \frac{1}{2}} - \hLW^{r - \frac{1}{2}} } + \norm{ \hLWlam^{1/2} \, (I - \glam(\hLW) \, \hLW) \, \hLW^{r - \frac{1}{2}} } \\
         & \le 2 F \LE( \lambda^{1/2} \cdot \norm{ \LT^{r - \frac{1}{2}} - \hLW^{r - \frac{1}{2}} } + \lambda^{r} \RI).
    \end{align*}
    Applying \lemref{lem: A^c - B^c}, we obtain
    \[
        \norm{ \LT^{r - \frac{1}{2}} - \hLW^{r - \frac{1}{2}} }
        \le \begin{cases}
            \norm{ \LT - \hLW }^{r - \frac{1}{2}},                              & r \in [1/2, 3/2]; \\
            (r - 1/2) \kappa^{2 r - 3} \nU^{\nu (r - 3/2)} \norm{ \LT - \hLW }, & r > 3/2,
        \end{cases}
    \]
    where we use the bound \( \norm{ \hLW } \le \kappa^{2} \nU^{\nu} \) with \( \nu = 2 \iota / (m (2 \iota + 1)) \). Next, we bound the operator norm \( \norm{ \LT - \hLW } \):
    \begin{align*}
        \norm{ \LT - \hLW }
         & \le \norm{ \LT - \LW } + \norm{ \LW - \hLW }                                                                                                                                       \\
         & \le \kappa^{2} \LE( (1 - \alpha)^{-3/2} \varDelta_{\phi} + \varXi_{m}^{1/2} \RI) \nU^{-\LE( \frac{\iota}{2 \iota + 1} - 2 \nu \RI)} \log \frac{18}{\delta'} + \norm{ \LW - \hLW },
    \end{align*}
    where the bound for \( \norm{ \LT - \LW } \) comes from \eqref{eq: LW - LT}. For \( \norm{ \LW - \hLW } \), define the random variable \( \xi(x) = \htheta(x) \, \Kx \otimes \Kx \) for \( x \sim \rhoSX \), with \( \xi_{i} = \xi(x_{i}) \). Its Hilbert-Schmidt norm is bounded by \( \kappa^{2} \nU^{\nu} \). Treating \( \xi \) as an element in the Hilbert space \( \mathscr{H} \) of Hilbert-Schmidt operators and applying \lemref{lem: concentration vector}, we have with probability at least \( 1 - \delta \),
    \[
        \norm{ \LW - \hLW } \le \norm{ \LW - \hLW }_{\mathscr{H}} \le 10 \kappa^{2} \nU^{\nu} \nL^{-1/2} \log \frac{4}{\delta}.
    \]
    Substituting these bounds gives
    \begin{align*}
         & \eqspace \norm{ \LT - \hLW }                                                                                                                                                                                                         \\
         & \le \kappa^{2} \LE( (1 - \alpha)^{-3/2} \varDelta_{\phi} + \varXi_{m}^{1/2} + 10 \RI) \LE( \nU^{-\LE( \frac{\iota}{2 \iota + 1} - 2 \nu \RI)} + \nU^{\nu} \nL^{-1/2} \RI) \LE( \log \frac{18}{\delta'} + \log \frac{4}{\delta} \RI).
    \end{align*}
    Consequently,
    \begin{align*}
         & \eqspace \norm{ \LT^{r - \frac{1}{2}} - \hLW^{r - \frac{1}{2}} }                                                                                                                                                                                                                                                       \\
         & \le \kappa^{2 r - 1} \LE( (1 - \alpha)^{-3/2} \varDelta_{\phi} + \varXi_{m}^{1/2} + 10 \RI) r                                                                                                                                                                                                                          \\
         & \eqspace \cdot \nU^{\nu \cdot \max \family{ r - \frac{3}{2}, 0 }} \LE( \nU^{-\LE( \frac{\iota}{2 \iota + 1} - 2 \nu \RI)} + \nU^{\nu} \nL^{-1/2} \RI)^{\min \family{ r - \frac{1}{2}, 1 }} \LE( \log \frac{18}{\delta'} + \log \frac{4}{\delta} \RI)                                                                   \\
         & \le \kappa^{2 r - 1} \LE( (1 - \alpha)^{-3/2} \varDelta_{\phi} + \varXi_{m}^{1/2} + 10 \RI) r                                                                                                                                                                                                                          \\
         & \eqspace \cdot \LE( \nU^{\nu \cdot \max \family{ r - \frac{3}{2}, 0 } - \min \family{ r - \frac{1}{2}, 1 } \cdot \LE( \frac{\iota}{2 \iota + 1} - 2 \nu \RI)} + \nU^{\nu \LE( \max \family{ r - \frac{3}{2}, 0 } + \min \family{ r - \frac{1}{2}, 1 } \RI)} \nL^{-\frac{1}{2} \min \family{ r - \frac{1}{2}, 1 }} \RI) \\
         & \eqspace \cdot \LE( \log \frac{18}{\delta'} + \log \frac{4}{\delta} \RI)                                                                                                                                                                                                                                               \\
         & = \kappa^{2 r - 1} \LE( (1 - \alpha)^{-3/2} \varDelta_{\phi} + \varXi_{m}^{1/2} + 10 \RI) r \lambda^{r - \frac{1}{2}} (\nU^{C_{1}} \nL^{C_{2}} + \nU^{C_{3}} \nL^{C_{4}}) \LE( \log \frac{18}{\delta'} + \log \frac{4}{\delta} \RI).
    \end{align*}
    The second inequality follows from the bound \( (a + b)^{c} \le a^{c} + b^{c} \) for any \( a, b > 0 \) and \( 0 \le c \le 1 \). The exponent terms \( C_{1}, C_{2}, C_{3}, C_{4} \) governing the sample complexity in the final expression are given by
    \begin{align*}
        \nU^{C_{1}} \nL^{C_{2}}
         & = \nU^{\nu \cdot \max \family{ r - \frac{3}{2}, 0 } - \min \family{ r - \frac{1}{2}, 1 } \cdot \LE( \frac{\iota}{2 \iota + 1} - 2 \nu \RI)} \nL^{s \LE( r - \frac{1}{2} \RI)}       \\
         & = \nU^{-\frac{\iota}{2 \iota + 1} \cdot \min \family{ r - \frac{1}{2}, 1 } + \nu \LE( \min \family{ r - \frac{1}{2}, 1 } + r - \frac{1}{2} \RI)} \nL^{s \LE( r - \frac{1}{2} \RI)},
    \end{align*}
    and
    \begin{align*}
        \nU^{C_{3}} \nL^{C_{4}}
         & = \nU^{\nu \LE( \max \family{ r - \frac{3}{2}, 0 } + \min \family{ r - \frac{1}{2}, 1 } \RI)} \nL^{s \LE( r - \frac{1}{2} \RI) - \frac{1}{2} \min \family{ r - \frac{1}{2}, 1 }} \\
         & = \nU^{\nu \LE( r - \frac{1}{2} \RI)} \nL^{\LE( r - \frac{1}{2} \RI) \LE( - \frac{1}{2} \min \family{ 1, \frac{2}{2 r - 1} } + s \RI)}                                           \\
         & \le \nU^{2 \nu \LE( r - \frac{1}{2} \RI)} \nL^{\LE( r - \frac{1}{2} \RI) \LE( - \frac{1}{2} \min \family{ 1, \frac{2}{2 r - 1} } + s \RI)},
    \end{align*}
    respectively. Under the condition \( \lambda = \nL^{-s} \) with \( 0 < s < \min \family{ 1 / 2, 1 / (2 r - 1) } \), if the sample sizes \( \nU \) and \( \nL \) satisfy
    \[
        \LE\{
        \begin{aligned}
             & \nL^{s} \le \nU^{\frac{\iota}{2 \iota + 1} \cdot \min \family{ 1, \frac{2}{2 r - 1} } - \nu \LE( \min \family{ 1, \frac{2}{2 r - 1} } + 1 \RI)}, \\
             & \nU^{2 \nu} \le \nL^{\frac{1}{2} \min \family{ 1, \frac{2}{2 r - 1} } - s},
        \end{aligned}
        \RI.
    \]
    we obtain the bound
    \begin{equation}
        \label{eq: sample bound 5}
        \lambda^{r - \frac{1}{2}} (\nU^{C_{1}} \nL^{C_{2}} + \nU^{C_{3}} \nL^{C_{4}}) \le 2 \lambda^{r - \frac{1}{2}}.
    \end{equation}
    Using the above estimates, we find
    \begin{align*}
        Q_{5}
         & \le 2 E F \norm{ \urho }_{\rhoTX} \LE( \lambda^{1/2} \cdot \norm{ \LT^{r - \frac{1}{2}} - \hLW^{r - \frac{1}{2}} } + \lambda^{r} \RI)                                                                             \\
         & \le 2 E F \norm{ \urho }_{\rhoTX} \LE( 2 \kappa^{2 r - 1} \LE( (1 - \alpha)^{-3/2} \varDelta_{\phi} + \varXi_{m}^{1/2} + 10 \RI) r \LE( \log \frac{18}{\delta'} + \log \frac{2}{\delta} \RI) + 1 \RI) \lambda^{r}
    \end{align*}
    with probability at least \( 1 - \delta \).
\end{proof}

We now proceed to prove \thmref{thm: importance weighting}.
\begin{proofof}{\thmref{thm: importance weighting}}
    First, set the regularization parameter as \( \lambda = \nL^{-s} \) with
    \[
        0 < s < \frac{1}{2 r + 1},
    \]
    which ensures that all constraints on the exponent \( s \) derived in the preceding propositions are satisfied. Setting \( \delta' = \delta / 2 \), the density ratio error bound in \eqref{eq: theta bound} holds with probability at least \( 1 - \delta / 2 \):
    \[
        \norm{ \theta - \htheta }_{\rhoSX} \le \LE( (1 - \alpha)^{-3/2} \varDelta_{\phi} + \varXi_{m}^{1/2} \RI) \nU^{-\LE( \frac{\iota}{2 \iota + 1} - 2 \nu \RI)} \log \frac{36}{\delta},
    \]
    where \( \nu = 2 \iota / (m (2 \iota + 1)) \).

    Next, assign a probability of \( \delta / 6 \) to each of \proref{pro: Q3}, \proref{pro: Q4}, and \proref{pro: Q5}. The sample size conditions derived from these propositions, along with the condition from \proref{pro: Q2}, can be presented as follows:
    \begin{itemize}
        \item \proref{pro: Q2}: \(
              \displaystyle
              \nU^{\frac{\iota}{2 \iota + 1} - 2 \nu} \ge 2 \kappa^{2} \LE( (1 - \alpha)^{-3/2} \varDelta_{\phi} + \varXi_{m}^{1/2} \RI) \log \frac{36}{\delta} \cdot \nL^{s};
              \)

        \item \proref{pro: Q3}: \(
              \displaystyle
              \LE\{
              \begin{aligned}
                   & \nU^{\frac{\iota}{2 \iota + 1} - 2 \nu} \ge \frac{2 \kappa^{2}}{\norm{ \LT }} \LE( (1 - \alpha)^{-3/2} \varDelta_{\phi} + \varXi_{m}^{1/2} \RI) \log \frac{36}{\delta}, \\
                   & \nL^{1 - s (2 r + 1)} \ge 14 \kappa^{2} \LE( \log \frac{24 \kappa^{2} (\norm{ \LT } + 2)}{\norm{ \LT } \delta} + 1 + \frac{1}{2 r} \RI) \nU^{2 \nu};
              \end{aligned}
              \RI.
              \)

        \item \proref{pro: Q4}: \(
              \displaystyle
              \LE\{
              \begin{aligned}
                   & \nU^{\nu} \le \nL^{\frac{1}{2} - s \LE( r + \frac{1}{2} \RI)}, \\
                   & \nL^{s} \le \nU^{\frac{\iota}{2 \iota + 1} - 2 \nu};
              \end{aligned}
              \RI.
              \)

        \item \proref{pro: Q5}: \(
              \displaystyle
              \LE\{
              \begin{aligned}
                   & \nL^{s} \le \nU^{\frac{\iota}{2 \iota + 1} \cdot \min \family{ 1, \frac{2}{2 r - 1} } - \nu \LE( \min \family{ 1, \frac{2}{2 r - 1} } + 1 \RI)}, \\
                   & \nU^{2 \nu} \le \nL^{\frac{1}{2} \min \family{ 1, \frac{2}{2 r - 1} } - s}.
              \end{aligned}
              \RI.
              \)
    \end{itemize}
    These conditions can all be implied by
    \begin{align}
        \label{eq: E1}
        \nU^{\frac{\iota}{2 \iota + 1} - 2 \nu} & \ge \frac{2 \kappa^{2}}{\min \family{ \norm{ \LT }, 1 }} \LE( (1 - \alpha)^{-3/2} \varDelta_{\phi} + \varXi_{m}^{1/2} \RI) \log \frac{36}{\delta} \cdot \nL^{s},
        \tag{\textbf{E1}}                                                                                                                                                                                          \\
        \label{eq: E2}
        \nL^{1 - s (2 r + 1)}                   & \ge 14 \kappa^{2} \LE( \log \frac{24 \kappa^{2} (\norm{ \LT } + 2)}{\norm{ \LT } \delta} + 1 + \frac{1}{2 r} \RI) \nU^{2 \nu},
        \tag{\textbf{E2}}
    \end{align}
    and
    \begin{alignat}{2}
        \label{eq: E3}
        \nL^{s}     & \le \nU^{\frac{\iota}{2 \iota + 1} \cdot \frac{2}{2 r - 1} - 2 \nu} \quad & \text{(when \( r > 3 / 2 \))},
        \tag{\textbf{E3}}                                                                                                        \\
        \label{eq: E4}
        \nU^{2 \nu} & \le \nL^{\frac{1}{2 r - 1} - s} \quad                                     & \text{(when \( r > 3 / 2 \))}.
        \tag{\textbf{E4}}
    \end{alignat}
    Furthermore, by relaxing these conditions, we can obtain a simpler form:
    \begin{alignat*}{2}
        \text{\eqref{eq: E1} \& \eqref{eq: E3}:}
        \quad &  & \nU^{\frac{\iota}{2 \iota + 1} \cdot \min \family{ 1, \frac{2}{2 r - 1} } - 2 \nu} & \ge \frac{2 \kappa^{2}}{\min \family{ \norm{ \LT }, 1 }} \LE( (1 - \alpha)^{-3/2} \varDelta_{\phi} + \varXi_{m}^{1/2} \RI) \log \frac{36}{\delta} \cdot \nL^{s}, \\
        \text{\eqref{eq: E2} \& \eqref{eq: E4}:}
        \quad &  & \nL^{\frac{1}{2 r + 1} - s}                                                        & \ge 14 \kappa^{2} \LE( \log \frac{24 \kappa^{2} (\norm{ \LT } + 2)}{\norm{ \LT } \delta} + 1 + \frac{1}{2 r} \RI) \nU^{2 \nu}.
    \end{alignat*}

    Under these conditions, the bounds stated in \proref{pro: Q1}, \proref{pro: Q2}, \proref{pro: Q3}, \proref{pro: Q4}, and \proref{pro: Q5} hold simultaneously with probability at least \( 1 - \delta \). That is,
    \[
        Q_{1} \le F \norm{ \urho }_{\rhoTX} \lambda^{r - \frac{\gamma}{2}},
        \quad Q_{2} \le \sqrt{2} \lambda^{-\gamma/2},
        \quad Q_{3} \le \sqrt{2},
    \]
    and
    \begin{align*}
        Q_{4} & \le 2 \sqrt{2} E \Bigg( 4 \kappa \LE( \kappa F \norm{ \urho }_{\rhoTX} + 2 G_{0} + \sqrt{2} \sigma_{0} \RI) \log \frac{12}{\delta}                                                                                           \\
              & \hphantom{\eqspace 2 \sqrt{2} E \Bigg( \,} + F \norm{ \urho }_{\rhoTX} \LE( \LE( \kappa^{2} \LE( (1 - \alpha)^{-3/2} \varDelta_{\phi} + \varXi_{m}^{1/2} \RI) \log \frac{36}{\delta} \RI)^{1/2} + 1 \RI) \Bigg) \lambda^{r}, \\
        Q_{5} & \le 2 E F \norm{ \urho }_{\rhoTX} \LE( 2 \kappa^{2 r - 1} \LE( (1 - \alpha)^{-3/2} \varDelta_{\phi} + \varXi_{m}^{1/2} + 10 \RI) r \LE( \log \frac{36}{\delta} + \log \frac{12}{\delta} \RI) + 1 \RI) \lambda^{r},
    \end{align*}
    leading to
    \begin{align*}
        \norm{ \flam - \hflam }_{\ranH{\gamma}}
         & \le Q_{1} + Q_{2} Q_{3} (Q_{4} + Q_{5})                                                                                              \\
         & \le \varDelta_{f} \LE( (1 - \alpha)^{-3/2} + \varXi_{m}^{1/2} + 1 \RI) \lambda^{r - \frac{\gamma}{2}} \log \frac{36}{\delta}         \\
         & = \varDelta_{f} \LE( (1 - \alpha)^{-3/2} + \varXi_{m}^{1/2} + 1 \RI) \nL^{-\LE( r - \frac{\gamma}{2} \RI) s} \log \frac{36}{\delta},
    \end{align*}
    where
    \begin{equation}
        \label{eq: Delta_f}
        \begin{aligned}
            \varDelta_{f}
             & = F \norm{ \urho }_{\rhoTX}                                                                                                    \\
             & \eqspace + 4 \sqrt{2} E \LE( (5 \kappa^{2} + 1) F \norm{ \urho }_{\rhoTX} + 8 \kappa G_{0} + 4 \sqrt{2} \kappa \sigma_{0} \RI) \\
             & \eqspace + 4 E F \norm{ \urho }_{\rhoTX} (40 \kappa^{2 r - 1} r + 1)
        \end{aligned}
    \end{equation}
    is a constant independent of \( \alpha \), \( m \), \( \nU \), \( \nL \), or \( \delta \). Setting \( \gamma = 0 \) and \( \gamma = 1 \) respectively yields
    \[
        \norm{ \frho - \hflam }_{\rhoTX} \le \varDelta_{f} \LE( (1 - \alpha)^{-3/2} + \varXi_{m}^{1/2} + 1 \RI) \nL^{-r s} \log \frac{36}{\delta}
    \]
    and
    \[
        \norm{ \frho - \hflam }_{\calH} \le \varDelta_{f} \LE( (1 - \alpha)^{-3/2} + \varXi_{m}^{1/2} + 1 \RI) \nL^{-\LE( r - \frac{1}{2} \RI) s} \log \frac{36}{\delta}.
    \]
    This completes the proof of the theorem.
\end{proofof}

Finally, we provide a short proof of \colref{col: polynomial}.
\begin{proofof}{\colref{col: polynomial}}
    Recall that the sample sizes satisfy the relation \( \nU = \nL^{\beta} \) for some \( \beta \ge 1 \). For any given \( \varepsilon > 0 \), we choose the moment parameter \( m \) sufficiently large such that
    \[
        2 \beta \nu = 2 \beta \LE( \frac{1}{m} \cdot \frac{2 \iota}{2 \iota + 1} \RI) < \varepsilon.
    \]
    Under this choice, the sample size condition \eqref{eq: sample size condition 2} required by \thmref{thm: importance weighting} simplifies to
    \begin{align*}
        \nU^{\frac{\iota}{2 \iota + 1} \cdot \min \family{ 1, \frac{2}{2 r - 1} } - 2 \nu}
         & \ge \nL^{\beta \frac{\iota}{2 \iota + 1} \cdot \min \family{ 1, \frac{2}{2 r - 1} } - \varepsilon}                                                              \\
         & \ge \frac{2 \kappa^{2}}{\min \family{ \norm{ \LT }, 1 }} \LE( (1 - \alpha)^{-3/2} \varDelta_{\phi} + \varXi_{m}^{1/2} \RI) \log \frac{36}{\delta} \cdot \nL^{s}
    \end{align*}
    and
    \begin{align*}
        \nL^{\frac{1}{2 r + 1} - s}
         & \ge 14 \kappa^{2} \LE( \log \frac{24 \kappa^{2} (\norm{ \LT } + 2)}{\norm{ \LT } \delta} + 1 + \frac{1}{2 r} \RI) \nL^{\varepsilon} \\
         & \ge 14 \kappa^{2} \LE( \log \frac{24 \kappa^{2} (\norm{ \LT } + 2)}{\norm{ \LT } \delta} + 1 + \frac{1}{2 r} \RI) \nU^{2 \nu}.
    \end{align*}
    These inequalities are equivalent to the following two requirements:
    \begin{enumerate}
        \item The labeled sample size \( \nL \) is sufficiently large, specifically
              \begin{align*}
                  \nL
                  \ge \max \Bigg\{ & \LE( \frac{2 \kappa^{2}}{\min \family{ \norm{ \LT }, 1 }} \LE( (1 - \alpha)^{-3/2} \varDelta_{\phi} + \varXi_{m}^{1/2} \RI) \log \frac{36}{\delta} \RI)^{\LE. 1 \MID/ \LE( \beta \frac{\iota}{2 \iota + 1} \cdot \min \family{ 1, \frac{2}{2 r - 1} } - s - \varepsilon \RI) \RI.}, \\
                                   & \LE( 14 \kappa^{2} \LE( \log \frac{24 \kappa^{2} (\norm{ \LT } + 2)}{\norm{ \LT } \delta} + 1 + \frac{1}{2 r} \RI) \RI)^{\LE. 1 \MID/ \LE( \frac{1}{2 r + 1} - s - \varepsilon \RI) \RI.} \Bigg\}.
              \end{align*}

        \item The exponent \( s \) governing the regularization parameter \( \lambda = \nL^{-s} \) satisfies
              \[
                  s \le \beta \frac{\iota}{2 \iota + 1} \cdot \min \family{ 1, \frac{2}{2 r - 1} } - \varepsilon,
                  \quad s \le \frac{2}{2 r - 1} - \varepsilon.
              \]
    \end{enumerate}

    Consequently, when \( 1 \le \beta < 1 + 1 / (2 \iota) \), we set
    \[
        s = \LE\{
        \begin{alignedat}{2}
             & \beta \frac{\iota}{2 \iota + 1} - \varepsilon,
             &                                                                                           & \quad \frac{1}{2} \le r < \frac{1}{2} + \LE( \frac{1 + 1 / (2 \iota)}{\beta} - 1 \RI);                                                        \\
             & \frac{1}{2 r + 1} - \varepsilon,
             &                                                                                           & \quad \frac{1}{2} + \LE( \frac{1 + 1 / (2 \iota)}{\beta} - 1 \RI) \le r \le \frac{1}{2} + \LE( \frac{1 + 1 / (2 \iota)}{\beta} - 1 \RI)^{-1}; \\
             & \beta \frac{\iota}{2 \iota + 1} \cdot \min \family{ 1, \frac{2}{2 r - 1} } - \varepsilon,
             &                                                                                           & \quad r > \frac{1}{2} + \LE( \frac{1 + 1 / (2 \iota)}{\beta} - 1 \RI)^{-1}.
        \end{alignedat}
        \RI.
    \]
    When \( \beta \ge 1 + 1 / (2 \iota) \), we set
    \[
        s = \frac{1}{2 r + 1} - \varepsilon,
        \quad \forall r \ge \frac{1}{2}.
    \]
    With these choices, all preceding requirements are satisfied. Applying \thmref{thm: importance weighting} then yields the following convergence rates:
    \[
        \norm{ \frho - \hflam }_{\rhoTX} = O \LE( \lambda^{r} \log \frac{36}{\delta} \RI),
        \quad \norm{ \frho - \hflam }_{\calH} = O \LE( \lambda^{r - \frac{1}{2}} \log \frac{36}{\delta} \RI).
    \]
    This completes the proof.
\end{proofof}

%% file: text/5_acknowledgment.tex
\phantomsection
\addcontentsline{toc}{section}{Acknowledgments}

\section*{Acknowledgments}

The authors would like to acknowledge the financial support from the relevant funding agencies. Ren-Rui Liu and Zheng-Chu Guo were partially supported by the National Natural Science Foundation of China (Project No.\ 12271473). Jun Fan was partially supported by the Research Grants Council of Hong Kong (Project Nos.\ HKBU12302923 and HKBU12303024) and the Guangdong and Hong Kong Universities ``1+1+1'' Joint Research Collaboration Scheme (Project No. 2025A0505000007). Lei Shi was partially supported by the National Natural Science Foundation of China (Project No.\ 12571099).

%% file: text/6_lemma.tex
\phantomsection \label{sec: lemma}
\addcontentsline{toc}{section}{Appendix}

\section*{Appendix}

This appendix presents auxiliary lemmas referenced in \secref{sec: proof}. We begin by introducing a concentration inequality for Hilbert-Schmidt operators, which serves as a fundamental tool in our analysis.
\begin{lemma}
    \label{lem: concentration operator}
    Let \( A_{1}, \dots, A_{n} \) be a sequence of i.i.d. self-adjoint Hilbert-Schmidt operators on a separable Hilbert space \( \mathscr{H} \). Assume that \( \EE{ A_{1} } = 0 \) and \( \norm{ A_{1} }_{\mathscr{H} \to \mathscr{H}} \le U \) almost surely for some \( U > 0 \), where \( \norm{ \cdot }_{\mathscr{H} \to \mathscr{H}} \) denotes the operator norm on \( \mathscr{H} \). Let \( V \) be a positive trace-class operator such that \( \EE{ A_{1}^{2} } \preceq V \). Then for any \( \delta \in (0, 1) \), with probability at least \( 1 - \delta \), we have
    \[
        \norm{ \frac{1}{n} \sum_{i=1}^{n} A_{i} }_{\mathscr{H} \to \mathscr{H}}
        \le \frac{2 U h(V)}{3 n} + \sqrt{\frac{2 \norm{ V }_{\mathscr{H} \to \mathscr{H}} h(V)}{n}},
    \]
    where
    \[
        h(V) = \log \frac{4 \Tr{ V }}{\norm{ V }_{\mathscr{H} \to \mathscr{H}} \delta}.
    \]
\end{lemma}
\begin{proof}
    This result is a restatement of \citet[Lemma~5.2]{Guo2019OptimalRC}, which itself generalizes \citet[Theorem~7.7.1]{Tropp2015IntroductionMC} using the arguments from \citet{Minsker2017SomeEB}. Setting \( \xi_{i} = A_{i} \), \( L = U \), \( v = n \norm{ V }_{\mathscr{H} \to \mathscr{H}} \), and \( d = \Tr{ V } / \norm{ V }_{\mathscr{H} \to \mathscr{H}} \) in that lemma yields
    \[
        \norm{ \frac{1}{n} \sum_{i=1}^{n} A_{i} }_{\mathscr{H} \to \mathscr{H}}
        \le \frac{t}{n}
    \]
    for any \( t \ge v^{1/2} + L / 3 \), with probability at least
    \[
        1 - 4d \exp \family{ -\frac{t^{2}}{2(v + L t / 3)} }.
    \]
    Setting this probability equal to \( 1 - \delta \) and solving for \( t \) gives
    \[
        t = \frac{U h(V)}{3} + \sqrt{2 n \norm{ V }_{\mathscr{H} \to \mathscr{H}} h(V) + \frac{U^{2} h^{2}(V)}{9} },
        \quad h(V) = \log \frac{4 \Tr{ V }}{\norm{ V }_{\mathscr{H} \to \mathscr{H}} \delta}.
    \]
    Since \( h(V) \ge 1 \) follows from the inequality \( \Tr{ V } \ge \norm{ V }_{\mathscr{H} \to \mathscr{H}} \), the condition \( t \ge v^{1/2} + L / 3 \) is always satisfied. Substituting this expression for \( t \) into the operator norm bound yields
    \[
        \norm{ \frac{1}{n} \sum_{i=1}^{n} A_{i} }_{\mathscr{H} \to \mathscr{H}}
        \le \frac{t}{n} \le \frac{2 U h(V)}{3 n} + \sqrt{\frac{2 \norm{ V }_{\mathscr{H} \to \mathscr{H}} h(V)}{n}}
    \]
    with probability at least \( 1 - \delta \), which completes the proof.
\end{proof}

We also need the following concentration inequality of random elements in Hilbert spaces, which is an adopted version of \citet[Proposition~2]{Caponnetto2007OptimalRR}.
\begin{lemma}[\citealp{Fischer2020SobolevNL}, Theorem~26]
    \label{lem: concentration vector}
    Let \( \xi \) be a random variable taking values in a separable Hilbert space \( \mathscr{H} \). Suppose that there exist positive constants \( G \) and \( \sigma \) such that
    \[
        \EE{ \norm{ \xi }_{\mathscr{H}}^{\ell} } \le \frac{1}{2} \ell! G^{\ell-2} \sigma^{2},
        \quad \ell = 2, 3, \dots
    \]
    Then, for any i.i.d. sample \( \family{ \xi_{i} }_{i=1}^{n} \) and any \( \delta \in (0, 1) \), with probability at least \( 1 - \delta \),
    \[
        \norm{ \frac{1}{n} \sum_{i=1}^{n} \xi_{i} - \EE{ \xi } }_{\mathscr{H}}
        \le 4 \LE( \frac{G}{n} + \frac{\sigma}{\sqrt{n}} \RI) \log \frac{2}{\delta}.
    \]
\end{lemma}
\begin{remark}
    If \( \norm{ \xi }_{\mathscr{H}} \le U \) almost surely for some \( U > 0 \), then \lemref{lem: concentration vector} applies with parameters \( G = U \) and \( \sigma = \sqrt{2} U \), leading to
    \[
        \norm{ \frac{1}{n} \sum_{i=1}^{n} \xi_{i} - \EE{ \xi } }_{\mathscr{H}}
        \le 4 \LE( \frac{U}{n} + \frac{\sqrt{2} U}{\sqrt{n}} \RI) \log \frac{2}{\delta}
        \le \frac{10 U}{\sqrt{n}} \log \frac{2}{\delta}
    \]
    with probability at least \( 1 - \delta \).
\end{remark}

Next, \lemref{lem: cordes} (also known as the Cordes inequality) and \lemref{lem: A^c - B^c} provide bounds for powers of linear operators.
\begin{lemma}[\citealp{Cordes1987SpectralTL}, Lemma~5.1]
    \label{lem: cordes}
    Let \( A \) and \( B \) be positive bounded linear operators on a separable Hilbert space \( \mathscr{H} \). For any \( c \in [0, 1] \), the following inequality holds:
    \[
        \norm{ A^{c} B^{c} }_{\mathscr{H} \to \mathscr{H}} \le \norm{ A B }_{\mathscr{H} \to \mathscr{H}}^{c}.
    \]
\end{lemma}

\begin{lemma}[\citealp{Blanchard2010OptimalLR}, Lemma~E.3]
    \label{lem: A^c - B^c}
    Let \( A \) and \( B \) be positive self-adjoint operators on a separable Hilbert space \( \mathscr{H} \) such that \( \max \family{ \norm{ A }_{\mathscr{H} \to \mathscr{H}}, \norm{ B }_{\mathscr{H} \to \mathscr{H}} } \le U \). Then, for any \( c \ge 0 \),
    \[
        \norm{ A^{c} - B^{c} }_{\mathscr{H} \to \mathscr{H}} \le
        \begin{cases}
            \norm{ A - B }_{\mathscr{H} \to \mathscr{H}}^{c},       & c \le 1;
            \smallskip                                                         \\
            c U^{c-1} \norm{ A - B }_{\mathscr{H} \to \mathscr{H}}, & c > 1.
        \end{cases}
    \]
\end{lemma}

The following lemma is also crucial. Its proof relies on elementary calculations and is detailed in \citet[Lemma~A.6]{Liu2025SpectralAM}.
\begin{lemma}
    \label{lem: sup fraction}
    For any \( \lambda > 0 \) and \( c \in [0, 1] \),
    \[
        \sup_{t \ge 0} \frac{t^{c}}{t + \lambda} \le \lambda^{c-1}.
    \]
\end{lemma}